\documentclass[letterpaper]{article} 
\usepackage{aaai24}  
\usepackage{times}  
\usepackage{helvet}  
\usepackage{courier}  
\usepackage[hyphens]{url}  
\usepackage{graphicx} 
\usepackage{natbib}  
\usepackage{caption} 
\usepackage{algorithm}
\usepackage{amsmath,mathtools}
\usepackage{subcaption}
\usepackage{amsthm}

\theoremstyle{definition}

\newtheorem{theorem}{Theorem}
\newtheorem{assumption}{Assumption}

\newtheorem{lemma}{Lemma}

\ifx\Pifont\undefined \usepackage{pifont} \fi

\ifx\xB\undefined

\newcommand{\fB}{\mathbf{f}}

\newcommand{\xB}{\mathbf{x}}

\newcommand{\KB}{\mathbf{K}}

\newcommand{\betaB}{\boldsymbol{\beta}}

\newcommand{\muB}{\boldsymbol{\mu}}

\newcommand{\Dcal}{\mathcal{D}}
\newcommand{\Ecal}{\mathcal{E}}

\newcommand{\Ncal}{\mathcal{N}}

\newcommand{\Rcal}{\mathcal{R}}

\newcommand{\Xcal}{\mathcal{X}}
\newcommand{\Ycal}{\mathcal{Y}}

\newcommand{\betaBRV}{\mbox{\boldmath \Pifont{psy}b}}

\newcommand{\argmax}{\text{argmax}}

\newcommand{\GP}{\mathcal{GP}}

\fi

\usepackage{newfloat}
\usepackage{listings}
\DeclareCaptionStyle{ruled}{labelfont=normalfont,labelsep=colon,strut=off} 
\floatstyle{ruled}
\newfloat{listing}{tb}{lst}{}
\floatname{listing}{Listing}
\usepackage{graphicx} 
\usepackage{subcaption} 
\usepackage{algpseudocode}
\usepackage{dsfont}
\usepackage{comment}
\usepackage{booktabs}

\usepackage{placeins}

\title{Domain Invariant Learning for Gaussian Processes and Bayesian Exploration}
\author {
    Xilong Zhao\textsuperscript{\rm 1}\thanks{Xilong Zhao is currently a student at MoE Key Lab of Artificial Intelligence, AI Institute, Shanghai Jiao Tong University, Shanghai 200240, China.},
    Siyuan Bian\textsuperscript{\rm 1},
    Yaoyun Zhang \textsuperscript{\rm 1},
    Yuliang Zhang \textsuperscript{\rm 1},
    Qinying Gu \textsuperscript{\rm 2},
    Xinbing Wang \textsuperscript{\rm 1},
    Chenghu Zhou \textsuperscript{\rm 1},
    Nanyang Ye \textsuperscript{\rm 1}\thanks{Corresponding author.}
}
\affiliations {
    \textsuperscript{\rm 1}Shanghai Jiao Tong University, Shanghai, China\\
    \textsuperscript{\rm 2}Shanghai Aritificial Intelligence Laboratory, Shanghai, China

    zhaoxilong@sjtu.edu.cn, biansiyuan@sjtu.edu.cn, yaoyunzhang897@gmail.com, blu\_26@sjtu.edu.cn, guqinying@pjlab.org.cn, xwang8@sjtu.edu.cn, zhouchsjtu@gmail.com, ynylincoln@sjtu.edu.cn
}

\usepackage{bibentry}

\begin{document}

\maketitle

\begin{abstract}
Out-of-distribution (OOD) generalization has long been a challenging problem that remains largely unsolved. Gaussian processes (GP), as popular probabilistic model classes, especially in the small data regime, presume strong OOD generalization abilities. Surprisingly, their OOD generalization abilities have been under-explored before compared with other lines of GP research. In this paper, we identify that GP is not free from the problem and propose a domain invariant learning algorithm for Gaussian processes (DIL-GP) with a min-max optimization on the likelihood. DIL-GP discovers the heterogeneity in the data and forces invariance across partitioned subsets of data. We further extend the DIL-GP to improve Bayesian optimization's adaptability on changing environments. Numerical experiments demonstrate the superiority of DIL-GP for predictions on several synthetic and real-world datasets. We further demonstrate the effectiveness of the DIL-GP Bayesian optimization method on a PID parameters tuning experiment for a quadrotor. 
The full version and source code are available at: https://github.com/Billzxl/DIL-GP.
\end{abstract}

\section{Introduction}
Gaussian processes (GP) have been widely used as models for Bayesian nonparametrics methods in machine learning \cite{ NEURIPS2022_6dd16c88, NEURIPS2022_2b2bf329}. However, the extrapolation ability of GP especially in the long tail regime suffers from generalization problems, i.e. failing on the non-i.i.d test data \cite{PILARIO2022100036}. How to guarantee the generalization ability of GP on data sampled from out-of-distribution sets is of great importance, also known as the \emph{out-of-distribution (OOD) generalization.} Previous GP methods mitigate the OOD generalization problems often by hand-crafting problem-specific kernel functions to capture the inherent long-tail characteristics of the data. While these methods are effective, it is hard to transfer the success to other types of problems \cite{PILARIO2022100036, Schmidt2023, chen4166591compressible, NEURIPS2022_99c80ceb}. This can also lead to degenerated sample efficiency of Bayesian optimization when GP is used as the surrogate function \cite{lei2021bayesian}.

To improve OOD generalization, many domain generalization methods have been proposed to learn invariant representations across different domains. This can be achieved by invariant risk minimization (IRM) \cite{arjovsky2019invariant}, which attempts to tackle this issue by penalizing predictions based on the unstable spurious features in the data collected from different domains. Although methods like IRM seems promising for OOD generalization, data need to be manually split into different domains to enable them to minimize the performance discrepancies across domains. This fundamental restrictions prevent this algorithmic schemes to be practical. First, how data are split can largely influence the final prediction performances. Second, sometimes it is impossible to split the data in continuous environments, especially for Bayesian optimization scenario, where the data is sequentially sampled from trials, such as sensor data collected on a drone flight under turbelent stormy weather.

In this paper, we propose a novel approach, domain invariant learning for Gaussian processes (DIL-GP), to iteratively construct worst-case domains/environments splittings for IRM for learning GP parameters. The optimized partition for data samples enforce GP to generalize on the worst possible data distributions thus benefit the OOD generalization. Then, DIL-GP is further extended as the surrogate function for Bayesian optimization. We demonstrate that the combination of this adversarial partitioning and IRM is key for learning generalizable GP kernels. Numerical experiments on synthetic and real-world datasets shows the superiority of the proposed algorithmic scheme. 

\section{Related Work}
\subsection{Gaussian Processes}

Gaussian Process(GP), as a non-parametric Bayesian model that models the relationship between data by defining a stochastic process  has the characteristics of non-parametricity, flexibility, interpretability, and uncertainty quantification. GP has demonstrated utility on diverse domains, including but not limited to regression analysis~\cite{JMLR:v23:19-597,lalchand2022sparse, zhang2022byzantinetolerant,pmlr-v139-artemev21a}, classification tasks~\cite{NEURIPS2021_50d2e70c,NEURIPS2021_e7e69cdf,pmlr-v139-achituve21a} and optimization~\cite{NEURIPS2021_3cec07e9,pmlr-v139-cai21f}. To improve GP's generalization ability, previous works focus on improving GP's kernel functions, especially for computation efficiencies. For example, in \cite{cohen2022loglineartime}, a binary tree kernel is proposed to improve its computation efficiency \cite{pmlr-v162-lu22b}. In \cite{jorgensen2022bezier}, a Bézier kernel is introduced that requires only linear cost in both number of observations and input features. This scaling is achieved through the introduction of the Bézier buttress, which allows approximate inference without computing matrix inverses or determinants.
Another line of works focus on improving GP's generalization abilities for specific applications. Shared GP kernels are proposed for multiple users for federated learning \cite{NEURIPS2021_46d0671d}. In \cite{chen2023metalearning}, meta learning is used to learn adaptable GP kernel parameters for molecular property prediction. Differential equation interactions are modeled with GP for physics process predictions in \cite{pmlr-v162-long22a}. Despite their domain-specific successes, general method to improve GP's generalization abilities, especially on out-of-distribution data, is largely under-explored.

\subsection{Out-of-Distribution Generalization}
Out-of-distribution (OOD) generalization, the task of generalizing under distribution shifts, has been researched in many areas, such as computer vision \cite{ZhangZLWSX22, HuangWHLX22, WangYHSZ21, Niu0ZZ22, NEURIPS2022_023d94f4, NEURIPS2022_0b5eb45a, NEURIPS2022_372cb780, NEURIPS2022_4730d10b}, natural language processing \cite{yang2022glue, hendrycks2020pretrained, sun2022counterfactual}, and speech recognition \cite{lu2022out, adolfi2023successes, NEURIPS2022_4730d10b}. Previous works mainly focus on improving the out-of-distribution generalization ability for deep neural networks. For example, \cite{HuangWHLX22} propose a deep neural network training strategy that randomly drop the most prominent feature for image classification in each iteration.  \cite{NEURIPS2022_cd305fde} propose to ensemble multiple deep neural networks each trained on a subset of the data to improve their out-of-distribution generalization abilities on image classification tasks. While these methods provide insights into how to improve the OOD generalization abilities of deep neural networks, they are reliant on neural architectures not suitable for GP. Besides, typical OOD generalization methods relies on the domain labels to learn invariant features across domains to achieve OOD generalization while the domain labels are hard or impossible to get in some scenarios where GP are widely used \cite{arjovsky2019invariant}. For example, in the small data regime when data collection is expensive, obtaining domain labels will be even harder. Besides, many existing datasets containing domain shifts may come without providing domain labels. In applications such as Bayesian optimization, where the data are actively sampled , the domain labels are also hard to obtain.

\section{Methodology}
We aim to learn a function $f: \Xcal \rightarrow \Ycal$ from $\Xcal \subseteq \Rcal^{d}$ to $\Ycal \subseteq \Rcal$ given the training set with $N$ sample pairs $X=(x_1, \cdots, x_n) \in \Rcal^{n \times d}$ of $n$ inputs with $x_j \in \Rcal^{d}$ and corresponding outputs $y=(y_1, \cdots, y_n) \in \Rcal^{n}$. 

\subsection{Gaussian Process} Gaussian process is a stochastic process $f \sim \GP(\mu, k_\theta)$ with mean function $\mu$ and kernel $k_{\theta}: \Rcal^{d} \times \Rcal^{d} \rightarrow \Rcal$ \cite{GPML}. The kernel is parameterized with $\theta \in \Rcal^{\Theta}$. The set of function values $\fB = (f(x_1), \cdots, f(x_n))^{T} \sim \Ncal(\muB, \KB_{\theta})$ is a joint Gaussian with $\muB_j=\mu(x_j)$ and $\KB_{ij}=k_{\theta}(x_i, x_j)$. In order to predict the corresponding output $y^{*}$ given the input datum $x^{*}$, the output $p(y^{*}|X, y, x^{*})=\Ncal(\mu^{*}, K^{*})$ is a Gaussian with mean and variance given by:
\begin{align}
&\mu^{*} = \mu(x^{*}) + \KB_{\theta}^{*x} (\KB_{\theta}^{xx} + \sigma^2 I)^{-1} ( y - \mu(x^{*}))\\
&K^{*} = \KB_{\theta}^{**} - \KB_{\theta}^{*x}(\KB_{\theta}^{xx} + \sigma^2 I)^{-1} \KB_{\theta}^{x*}
\end{align}
where $\KB_{\theta}^{**} = [k(x^{*}, x^{*})]$ and $\KB_{\theta}^{*x}=[k(x^{*}, x_1), \cdots ,\\ k(x^{*}, x_{n})]^{T}$ denotes the kernel matrix and similarly for $\KB_{\theta}^{xx}$ and $\KB_{\theta}^{x*}$. $I$ is the identity matrix. The kernel parameters $\theta$ can largely determine the performance of the GP, and the optimal kernel parameters can be obtained by the maximum posterior likelihood method:
\begin{equation}
\begin{aligned}
\log p(y|x, \theta) = &-\frac{1}{2}(y-\mu(x))^{T} (\KB_{\theta}^{xx} + \sigma^2 I )^{-1} (y-\mu(x)) \\
&-\frac{1}{2} \log |\KB_{\theta}^{xx} +\sigma^2 I| - \frac{n}{2} \log(2\pi)
\end{aligned}
\end{equation}
where the first term measures the goodness of fit, the second term is the regularization for the kernel complexity, and the last term is the normalizing constant. 

\subsection{Data Distribution Shifts Challenges} 
Data distribution shifts widely exist in real world scenarios as it is practically impossible to collect data from all domains. For example, a computer disease diagnosis system trained in one hospital's data has to generalize to other hospitals. This also pose challenges for GPs, suppose data is sampled from different environments/domains $e \in \Ecal$, $x \sim p(x|e) $. The average negative likelihood (loss) on training and test domains are given below:
\begin{align}
    \text{NLL}_{\text{train}}(\theta)=- \int p_{\text{train}}(e) p(x|e) \log p(y|x, \theta, e)  
    \\ \text{NLL}_{\text{test}}(\theta)=- \int p_{\text{test}}(e) p(x|e) \log p(y|x, \theta, e)
\end{align}
It is obvious to see the minimizer of $\text{NLL}_{\text{train}}$ is not necessarily suitable for $\text{NLL}_{\text{test}}$. This poses challenges for GP's generalization abilities. To solve this problem, an intuitive way is to minimize the negative likelihood on different domains at the same time. For example, the distributional robustness optimization method seeks to optimize on the worst environment's data to achieve OOD generalization \cite{Sagawa2020Distributionally}. Invariant risk minimization enforce models to reach local minima for each domain to learn domain invariant representations \cite{arjovsky2019invariant}.  However, these approaches all require prior knowledge about which domain the data belongs to. Next we will present algorithm for inferring domains directly from data without human assigned environment prior.

\subsection{Inferring Domains for Domain Invariant Learning}
Intuitively, as data are sampled from the mixture of different domains, if the model can perform well on the worst possible domain, the model will be suitable for other domains. As the domain information is unknown in our setting, we partition the data into different subsets to construct domains for domain invariant learning. The data are partitioned to create challenges for invariant learning methods. Here we use the invariant risk minimization penalty as in \cite{arjovsky2019invariant} but other variants can also be used. We denote the domain index for each data as $q_i \in [0, \cdots, \Ecal]$ to represent which domain the data pair $(x_i, y_i)$  belongs to.  We first give the likelihood on environment $e$:

\begin{equation}
    \begin{aligned}
\log p(&y|x, \theta, e) =  -\frac{1}{2}((y-\mu(x))*\mathds{1}
(q=e))^{T} (\KB_{\theta}^{xx} + \sigma^2 I )^{-1} \\
&((y-\mu(x))*\mathds{1}
(q=e)) - \frac{1}{2} \log |\KB_{\theta}^{xx} + \sigma^2 I| - \frac{n}{2} \log(2\pi)  
\end{aligned}
\end{equation}

where $*$ is the element-wise multiplication operation,  $q=[q_1, q_2, \cdots, q_n]^{T}$, $\mathds{1}$ is the identity function. Then our objective is to maximize the invariant risk minimization penalty with regard to the domain index $q$:

\begin{align}
\max_{q} \sum_{e \in \text{supp}(e)} \|\nabla_{w}|_{w=1} \log p(y|x, w*\theta, e) \|^2
\label{eq:innermax}
\end{align}

Where $\text{supp}(e)$ represents the set of possible values for the environment variable $e$. As $q_i$ is a positive discrete variable, the above optimization problem is NP-hard problem. We fix $\Ecal$ to be two and use a soft subsitute for $\mathds{1}
(q=e)$: $\mathds{1}(q=e) \approx \text{sigmoid}(\tilde{q}), \tilde{q} \in \Rcal^{n}$. This facilities the gradient computation and leads to stable empirical performances.

\subsection{Domain Invariant Learning for Gaussian Processes}
After the domain labels are inferred, we finally derive the domain invariant learning for Gaussian processes (DIL-GP). The algorithm is shown in Algorithm~\ref{alg:DIL-GP}. Compared with the vanilla GP, our proposed algorithmic scheme is less prone to over-fitting to the majority group of the data and can provide fairer and better performances especially in novel domains.  As one of the most important applications of GP, Bayesian optimization is widely used for black-box optimization. The generalization ability of surrogate model on unseen data is pivotal for the success of Bayesian optimization \cite{hvarfner2022joint,NEURIPS2021_11704817,NEURIPS2021_a8ecbaba,NEURIPS2021_68521755}. We extend DIL-GP for Bayesian optimization to further demonstrate its OOD generalization ability.

\begin{algorithm}[!ht]
\caption{Domain Invariant Learning for Gaussian Processes}
\label{alg:DIL-GP}
\begin{algorithmic}[1]
\Require {Input data pairs $\{x_i, y_i\}, i=1,\cdots, n$, number of inner maximization steps $T_1$, number of outer minimization steps $T_2$, initial kernel functions $K_{\theta}(\cdot, \cdot)$ parameterized by $\theta$, inner maximization learning rate $\eta_1$, outer minimization learning rate $\eta_2$, IRM penalty coefficient $\lambda$.}
\Ensure {Optimized kernel parameters $\theta$} 
\State {Initialize kernel matrix $K_{\theta}^{xx}$ with input data $\{x_i, y_i\}, i=1,\cdots, n$}
\State {Initialize the domain index vector $\tilde{q}$}
\For {$t_{\text{inner}} = 1$ \textbf{to} $T_1$} 
\For {$t_{\text{outter}} = 1$ \textbf{to} $T_2$}
\State Update $\tilde{q}$ with gradient ascent according to Eq~\ref{eq:innermax}:
\begin{equation}
\tilde{q}_{t+1} \leftarrow \tilde{q}_{t} + \eta_1 \nabla_{\tilde{q}}\sum_{e \in \text{supp}(e)} \|\nabla_{w}|_{w=1} \log p(y|x, w*\theta, e) \|^2
\nonumber
\end{equation}
\EndFor
\State Update $\theta$ with gradient descent on the likelihood with the IRM penalty term:

\begin{equation}
\begin{aligned}
\theta_{t+1} \leftarrow &\theta_{t} - \eta_2  \nabla_{\theta} \Biggl[\log p(y|x,\theta) \\
&+ \lambda \sum_{e \in \text{supp}(e)} \|\nabla_{w}|_{w=1} \log p(y|x, w*\theta, e)  \|^2 \Biggr] \nonumber
\end{aligned}
\end{equation}

\EndFor
\State Output optimized kernel parameters $\theta$.
\end{algorithmic}
\end{algorithm}

\begin{theorem}
Under mild assumptions, given $\delta \in (0, 1)$, DIL-GP's  OoD risk is strictly no larger than vanilla GP's OoD risk $R_{\text{DIL-GP}}= \mathds{E}_{x^*}(\mu_{\text{DIL-GP}}(x^*)-f(x^*))^2 \leq R_{\text{GP}}= \mathds{E}_{x^*}(\mu_{\text{GP}}(x^*)-f(x^*))^2$, with probability $\geq 1-\delta$.
\label{theorem1}
\end{theorem}
The proof of this theorem can be found in the Appendix.

\begin{figure*}[!ht]
\centering
  \includegraphics[width=1\textwidth]{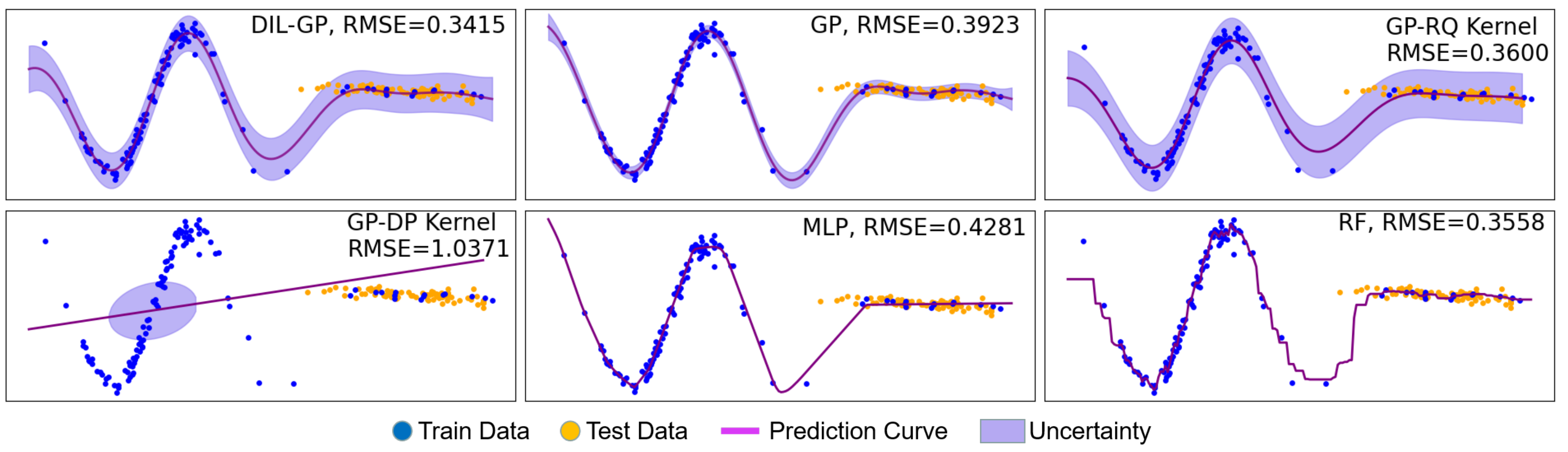}

  \caption{Results on the one-dim synthetic dataset.
  }
  \label{gp_syn_result}
\end{figure*}

\begin{table*}[!ht]
    \begin{center}
    {
        \begin{tabular}{l|c|c|c|c}
            \toprule 
             & \multicolumn{2}{c|}{\textbf{1-Dim Dataset}} & \multicolumn{2}{c}{\textbf{2-Dim Dataset}} \\
            \midrule
            Method & RMSE  & Coverage Rate & RMSE  & Coverage Rate \\
            \midrule
            Random Forest& 0.3558 $\pm$ 0.0840 & - & 0.7938 $\pm$ 0.1841 & -  \\
            MLP & 0.4321 $\pm$ 0.0427 & - & 0.9133 $\pm$  0.4151 & -  \\
            GP & 0.3923 $\pm$ 0.0153 &  0.9125 $\pm$ 0.0374& 0.8022 $\pm$  0.1373 &  0.9225 $\pm$ 0.0537 \\
            GP-RQ Kernel & 0.3600 $\pm$ 0.0064& 0.9625 $\pm$ 0.1256  & 0.7921 $\pm$ 0.2107& 0.8732 $\pm$ 0.1436 \\
            GP-DP Kernel & 1.0371 $\pm$ 0.0021& 0 $\pm$ 0& 0.7133 $\pm$ 0.0096& 0.6339 $\pm$ 0.2311   \\
            \midrule 
            \textbf{DIL-GP} & \textbf{0.3415 $\pm$ 0.0109} & 0.9625 $\pm$ 0.1000& \textbf{0.6583 $\pm$ 0.0230} & 0.8822 $\pm$ 0.0775  \\
            \bottomrule
        \end{tabular}

    }
\caption{Quantitative comparison between different methods on synthetic datasets (Mean $\pm$ Max deviation over 5 runs)}
\label{table:syn2}
\end{center}
\end{table*}{}

\subsection{Bayesian Optimization with DIL-GP}
To adapt surrogate model in Bayesian optimization in fast-changing environments, we propose to incorporate DIL-GP as the surrogate model for Bayesian optimizations. Previous Bayesian optimization methods tend to ignore the domain shifts existed in the sampled data thus the optimized black-box model may perform poorly in novel domains. The DIL Bayesian optimization algorithm is shown in Algorithm~\ref{alg:DILGP-BO}. Here, $D_t$ is a data set that includes all queried data points $(x_i, f(x_i))$ so far, $\alpha(x \mid D_t)$ is an acquisition function used to determine the next point to query given the current Gaussian process model. Commonly used acquisition functions include expected improvement (EI) and upper confidence bound (UCB). At each iteration, the algorithm fits a Gaussian process model to the current data set, selects the next point to query based on the acquisition function, and adds the new data point to the data set. The algorithm continues until a stopping criterion is met, and the optimal solution $\hat{x}$ is the point with the highest function value among all known data points.

\begin{algorithm}[!ht]
\caption{DIL-GP Bayesian Optimization}
\label{alg:DILGP-BO}
\begin{algorithmic}[1]
\Require {Black-box function to be optimized $f(\beta)$, Search space $\betaBRV$, Acquisition function $\alpha$,number of Bayesian optimization steps $T_{\text{BO}}$.}
\Ensure {Optimized solution $\beta$} 
\State {Randomly sample initial trial points $\betaB \sim \betaBRV$.}
\State {Compute the initial queried data point set $D_0 = {(\betaB,f(\betaB)}))$}
\For { $t = 0$  \textbf{to} $T_{\text{BO}}-1$} 
\State Use Algorithm~\ref{alg:DIL-GP} to fit a GP model: $\text{DIL-GP}(D_t)$;
\State Find the next trial point with $\text{DIL-GP}(D_t)$ using the acquisition function: 
\begin{equation}
\hat{\beta} = \arg\min_{\beta \in \betaBRV} \alpha(\beta \mid \text{DIL-GP}(D_t)) \nonumber
\end{equation}
\State Evaluate $f(\hat{\beta})$; 
\State $D_{t+1} \gets D_t \cup {(\hat{\beta},f(\hat{\beta}))}$;
\EndFor
\State Output the optimized solution $\hat{\beta} = \arg\min_{\beta\in D_t} f(\beta)$;
\end{algorithmic}
\end{algorithm}

We prove an upper bound on the cumulative regret of DIL-Bayesian optimization (DIL-BO), where the procedure for proving is in the Appendix:
\begin{theorem}
\textbf{(Convergence of DIL-BO)} Given $\delta \in (0,1)$, denote $\gamma_t$ the maximum information gain after observing $t$ observations. If run the BO process with $\beta_t=B+\sigma\sqrt{2(\gamma_{t-1}+1+\log(4/\delta)}$, with probability$\geq 1-\delta/4$, the upperbound of the cumulative regret $R_T$ satisfies:
\begin{equation}
    R_T = \sum_{t=1}^{T} r_t \leq \beta_T \sqrt{C_1T\gamma_T}
\end{equation}
\label{theorem2}
\end{theorem}

\section{Experiments}

In the experiments, we compare the proposed DIL-GP with GP with Gaussian kernel (GP), GP with the rational quadratic kernel (GP-RQ Kernel)\cite{mackay2003information}, GP with the dot product kernel (GP-DP Kernel)\cite{GPML}, random forest (RF)\cite{breiman2001random}, and multi-layer perceptron (MLP)\cite{rumelhart1986general}. For DIL-GP, we use the Gaussian kernel for all experiments to demonstrate the algorithms' generalization ability without the need for special kernel treatments. Detailed settings can be found in the Appendix.

\begin{table*}[!ht]
	\begin{center}
    {
        \small{
        \begin{tabular}{l|c|c|c|c|c}
            \toprule 
            & \multicolumn{4}{c|}{\textbf{King Housing Dataset}} & \multicolumn{1}{c}{\textbf{Automobile Dataset}} \\
            \midrule
            \textbf{Method} & domain1 & domain2 & domain3 & domain4& test dataset \\
            \midrule
            Random Forest& 0.450 $\pm$ 0.023 & 0.560 $\pm$ 0.031 & 0.623 $\pm$ 0.106 & 0.890 $\pm$ 0.037  & 0.992 $\pm$ 0.206\\
            MLP &  0.536 $\pm$ 0.137 & 0.650 $\pm$ 0.347 & 0.787 $\pm$ 0.762 & 0.913 $\pm$ 0.989 & 0.983 $\pm$ 0.072 \\
            GP &  0.371 $\pm$ 0.009 & 0.594 $\pm$ 0.056 & 0.800 $\pm$ 0.084 & 0.989 $\pm$ 0.100  & 0.889 $\pm$ 0.008 \\
            GP - RQ Kernel &  0.365 $\pm$ 0.016 & 0.442 $\pm$ 0.093 & 0.604 $\pm$ 0.173 & 0.901 $\pm$ 0.075 & 0.847 $\pm$ 0.351 \\
            GP - DP Kernel &  0.556 $\pm$ 0.007 & 0.663 $\pm$ 0.006 & 0.812 $\pm$ 0.010 & 1.056 $\pm$ 0.012  & 0.938 $\pm$ 0.013\\
            \midrule 
            \textbf{DIL-GP} & \textbf{0.291 $\pm$ 0.016}& \textbf{0.371 $\pm$ 0.011}  & \textbf{0.489 $\pm$ 0.042} & \textbf{0.543 $\pm$ 0.058} & \textbf{0.837 $\pm$ 0.069}  \\
            \bottomrule
        \end{tabular}
        }
    }
\caption{Quantitative comparison of RMSE between different methods on King Housing Dataset and  Automobile Dataset (Mean $\pm$ Max deviation over 5 runs)}
\label{table:compare_gp_real1}
\end{center}
\end{table*}{}

\begin{table*}[!ht]
	\begin{center}
    {
        \begin{tabular}{l|c|c|c|c}
            \toprule 
            Method & Hover  & Fig-8 & Sin-forward & Spiral-up\\
            \midrule
            Random Forest & 0.3740$\pm$0.0273 & 0.4444$\pm$0.0252  &{0.3962$\pm$0.0036 } & {0.3871$\pm$0.0310 } \\
            MLP  & 0.3735$\pm$0.0391 & 0.4433$\pm$0.0517 & 0.3582$\pm$0.0058 & 0.3733$\pm$0.0399  \\
            GP & 0.3739$\pm$0.0187  &  0.4433$\pm$0.0241   & 0.3579$\pm$0.0032 & 0.3727$\pm$0.0333  \\
            GP-RQ Kernel & 0.3733$\pm$0.0213 & 0.4492$\pm$0.0293 & 0.3572$\pm$0.0040  & 0.3724$\pm$0.0283  \\
            GP-DP Kernel & 0.3731$\pm$0.0238 & 0.4665$\pm$0.0371  & 0.3469$\pm$0.0052 & 0.3474$\pm$0.0222 \\
            \midrule 
            \textbf{DIL-GP} & \textbf{0.3610 $\pm$ 0.0200} & \textbf{0.4081 $\pm$ 0.0262}& \textbf{0.3367 $\pm$ 0.0043}& \textbf{0.3349 $\pm$ 0.0275} \\
            \bottomrule
        \end{tabular}
    }
    \caption{Errors between the actual trajectory and the desired trajectory  (Mean $\pm$ Max deviation over 5 runs)}
\label{table:bo}
\end{center}
\end{table*}

\subsection{Synthetic Dataset}

We first generate a one-dim synthetic dataset to demonstrate the effectiveness of our model. This dataset comprises two clusters of data with Gaussian distributions, which simulate two distinct domains. Specifically, the first cluster $X_1 \sim \mathcal{N}(0, 1)$, and the second cluster $X_2 \sim \mathcal{N}(6.5, 1)$. Our training set consists of one hundred samples from cluster $X_1$ and fifteen samples from cluster $X_2$ while the test set consists of eighty samples from cluster $X_2$. This setting generates a typical non-i.i.d distribution shifts scenario. The rooted mean square error (RMSE) is used as the comparison metric. For GP-related methods, we define another uncertainty metric---coverage rate, which represents the proportion of test set data that falls within the standard deviation region for GP models. When RMSE values are similar, higher coverage rates on the test set show better uncertainty discovery. 
Our results, presented in Figure~\ref{gp_syn_result}, reveal that GP baselines, MLP, and RF all overfit on the first cluster. Moreover, the GP baselines exhibit overconfidence on the data, without considering the heterogeneity, leading to overfitting and worse test set performances, as demonstrated in Table~\ref{table:syn2}. The GP-DP kernel fails to fit the data, indicating that changing kernels may improve GP's generalization performances on special applications but may not be generalizable to other domains. 

We further demonstrate the effectiveness of our model on a two-dimensional dataset comprising more complex functions. Specifically, we sample two distinct clusters from Gaussian distributions with varying means and variances, and generate labels using a trigonometric function (Details in Appendix). The first cluster $X_1 \sim \mathcal{N}(\begin{bmatrix}0.3 \\ 0.3\end{bmatrix}, \begin{bmatrix}0.01 & 0 \\ 0 & 0.01\end{bmatrix})$, and the second cluster $X_2 \sim \mathcal{N}(\begin{bmatrix}0.7 \\ 0.7\end{bmatrix}, \begin{bmatrix}0.01 & 0 \\ 0 & 0.01\end{bmatrix})$.The training set consists of one hundred samples from cluster one and fifteen samples from cluster two, while the test set solely comprises eighty samples from cluster two. We benchmark all methods to predict the two-dimensional data and present the experimental results in Table~\ref{table:syn2}. The experimental results on the 2D dataset show that  MLP suffers great performance degradation when facing heterogeneous data. RF, GP baseline, GP-RQ kernel and GP-DP kernel also pose significant gaps with DIL-GP due to the lack of domain generalization capability. 

In the two-dimensional generative dataset, DIL-GP shows more obvious advantages compared to the one-dimensional data. This may be due to the fact that higher dimensional data domains contain synergistic relationships with each other, while the domain invariant learning in DIL-GP learns more robust representations for OoD generalization. Subsequently, we conduct experiments on more complex real-world dataset to check its performance.

\subsection{Real-world Datasets}

We test the practicability of the algorithm on real-world datasets. The first dataset we consider is a regression dataset (Kaggle) of house sales prices from King County, USA. In this house prices predicting dataset, houses built in different periods are considered as different domains. Houses built in different periods possess distinct construction materials  and are affected by the cutural heritage values. This complicates the relationship between domain and selling price. Therefore, solving the out-of-distribution problem is a challenging task.

We sample a dataset of size 1304 from it, using the 17 variables including the house size, number of floors and the latitude and longitude locations to predict the transaction prices of the house. To simulate non-i.i.d. distribution shifts, we split the dataset according to the built year of the houses.
We divide the dataset into five domains, using data with build year from 1980 to 2015 as the training dataset (D0) and build year from 1900 to 1979 as test dataset. The test dataset is further divided into four domains with a span of 20 years (D1-D4). The motivation for selecting earlier-built houses as the test set is that the value of houses built in earlier periods is harder to predict, making the problem more challenging \cite{klein1975our}. Our results, presented in Figure~\ref{gp_real1_result}, demonstrate that DIL-GP achieves good results especially in domains with a long time interval from the training set. DIL-GP is also more stable across multiple domains, demonstrating its effectiveness in OOD generalization. MLP outperforms DIL-GP in some experiments but suffers significant instability. We provide quantitative results in Table~\ref{table:compare_gp_real1}.

The second real-world dataset under consideration is the Automobile dataset from UCI datasets with more than one hundred samples with fourteen dimensions' features (including fuel-type, aspiration, body-style, compression-ratio, etc) for insurance risk rating of automotive types. The term "insurance rating" refers to the extent to which an automotive type is more risky than its market price indicates in actuarial science. Different types of automobiles have implications for various usage scenarios and demand-supply relation ships and so on. It is reasonable to consider the types of automobiles as different domains. However, it is hard to make a clear statement which specific types of cars should be grouped in the same domain to train a good predictor, making it a challenging problem for previous methods reliant on domain labels. For evaluation, we use the "sedan" and "hardtop" types as the training domains, while "wagon", "hatchback", and "convertible" types are for test domains, which we do not provide for training algorithms. As shown in Table~\ref{table:syn2}, although the domain labels are not provided in the training as previous OOD generalization methods, DIL-GP attains the lowest RMSE on the novel domains. This further verifies DIL-GP's min-max formulations' adaptability for practical problems and its OOD generalization abilities under diverse kinds of distribution shifts.

\begin{figure}[!ht]
\centering
 \includegraphics[width=0.5\textwidth]{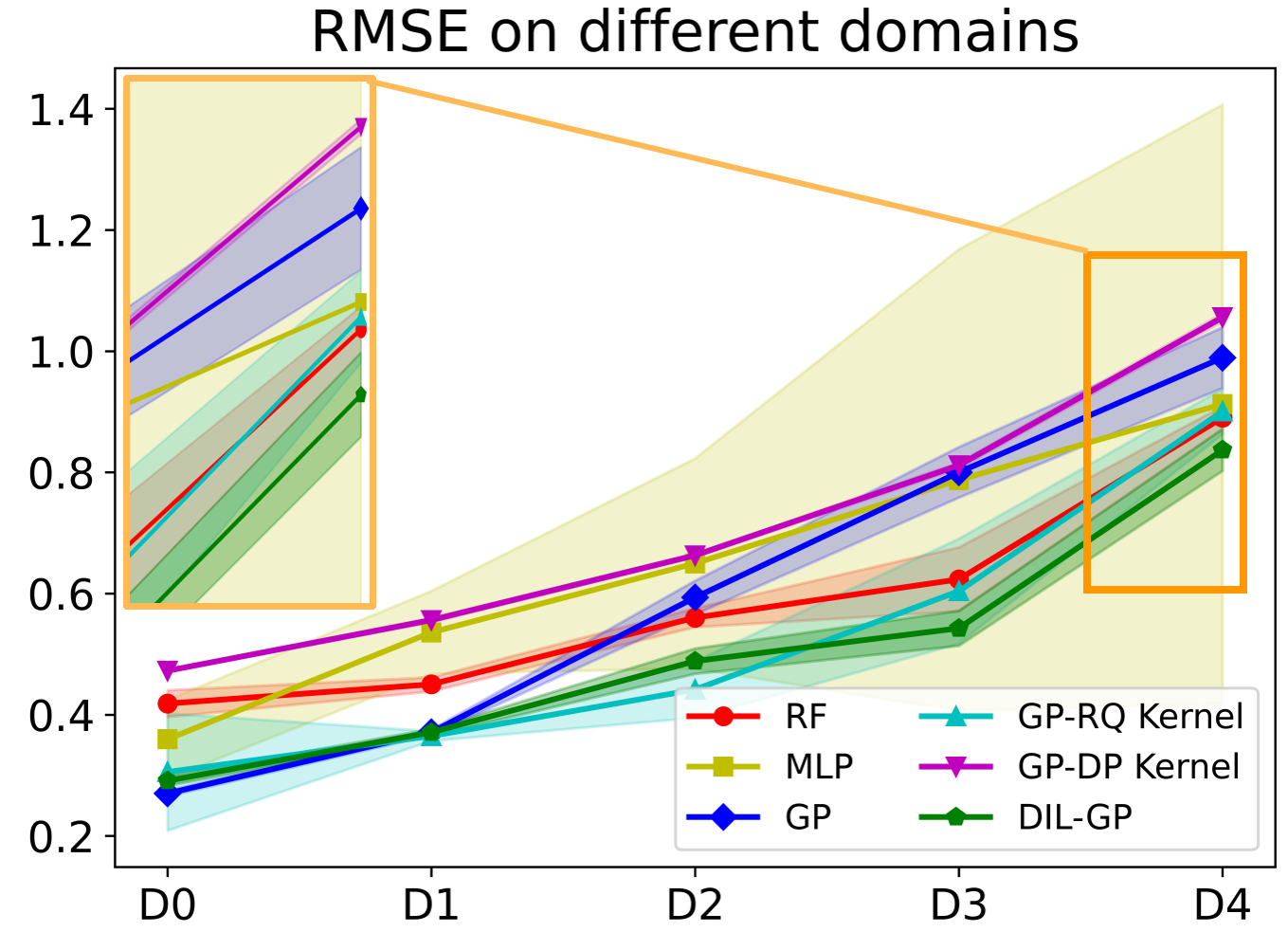}
    \caption{RMSE of methods on different domains (D0-D4) (Mean $\pm$ $\frac{1}{2}$ Max deviation over 5 runs).}
  \label{gp_real1_result}
\end{figure}

\begin{figure}[!ht]
  \includegraphics[width=0.5\textwidth]{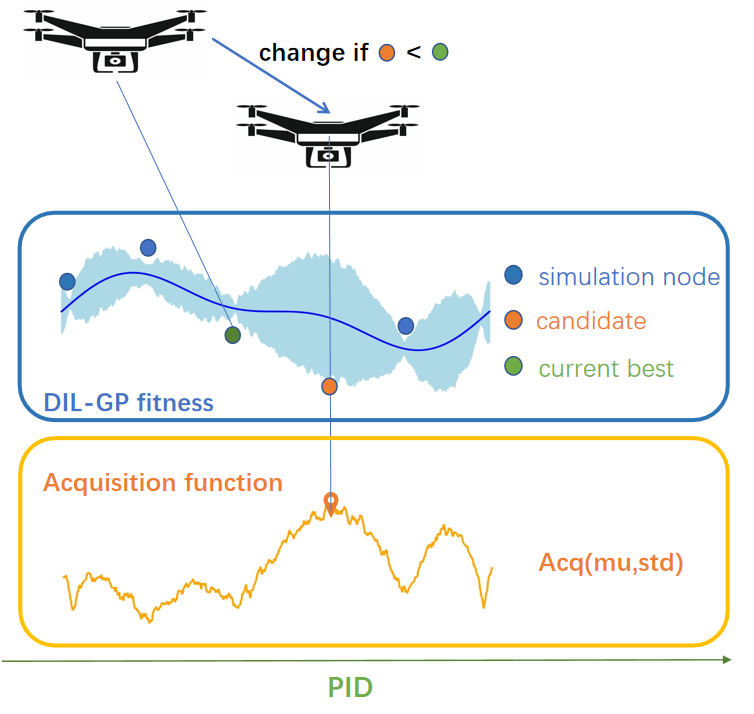}
  \caption{Bayesian optimization for PID tuning. A GP function is fitted to the existing data points, and a scoring function is used to compute scores for various parameter settings. The highest-scoring parameter setting is selected as the candidate PID and added to the data points. If this candidate outperforms the current best, it becomes the new best. This process is iteratively repeated until the termination condition is met.}
  \label{bo_pic}
\end{figure}

\begin{figure*}[!ht]
  \centering
  \includegraphics[width=0.88\textwidth]{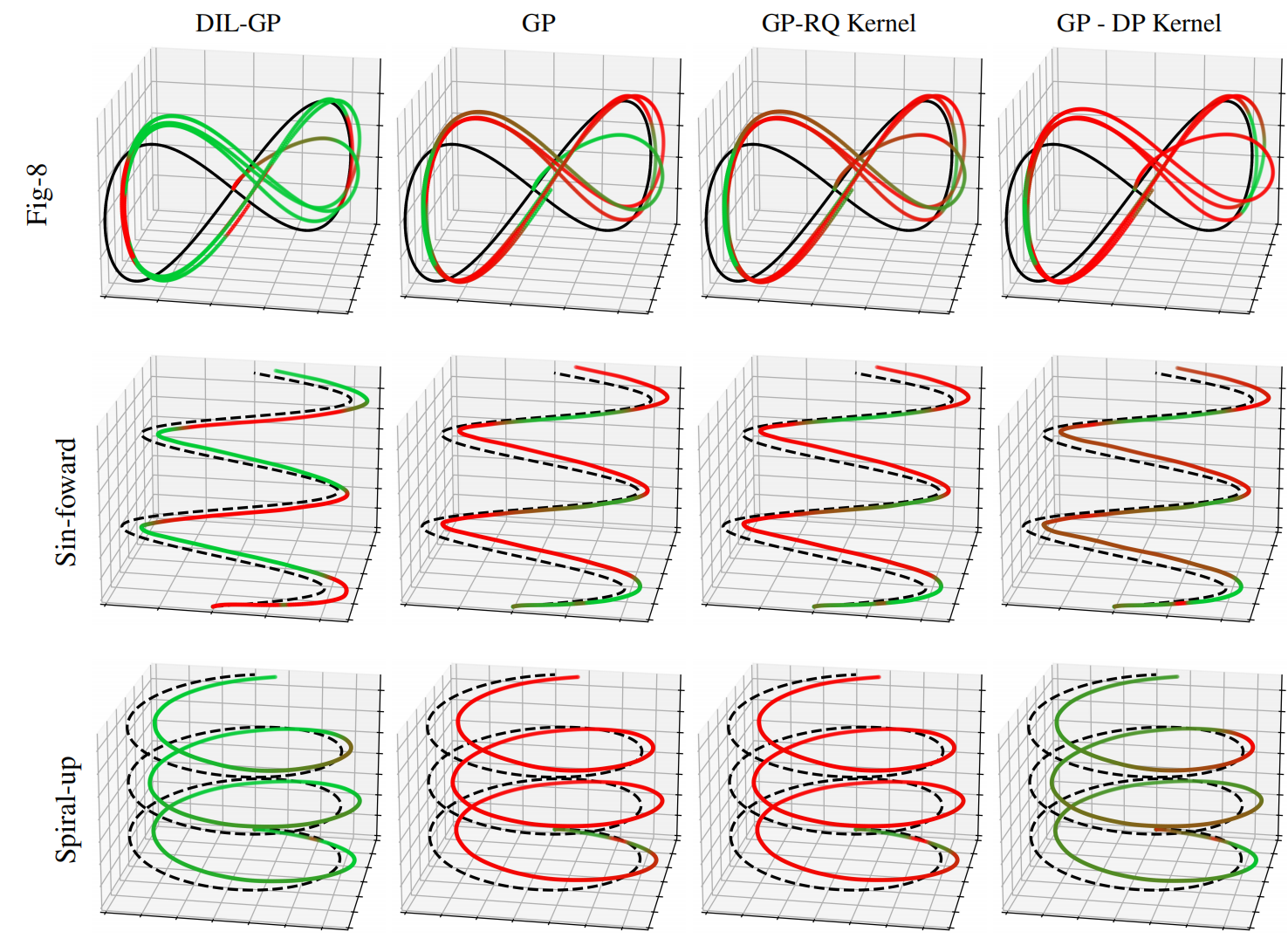}
  \caption{Visualizations of the quadrotor flight simulations on various trajectories. The black line is the desired trajectory. The green line denotes lower deviation from the desired trajectory while the red line denotes farther distances. BO using DIL-GP finds the best PID parameters, which enable accurate controls across all trajectories types under challenging turbulent wind conditions.}
  \label{uav_result}
\end{figure*}

\subsection{Bayesian Optimization for Quadrotors PID Tuning}
The Proportional-Integral-Derivative (PID) control is a widely used feedback control algorithm that has been applied in various applications, such as quadrotor control \cite{johnson2005pid}. This control algorithm employs the weighted feedback of the errors, consisting of three components: proportional, integral, and derivative of errors.  PID control takes in sensor data and determines the motor speeds of rotors for quadrotor flight control. Tuning the PID parameters is tedious but challenging due to the complex dynamics, sensitivity to changes and external influences. It usually requires human experts to achieve optimum performance: the weights for the three components need to be carefully tuned to achieve stable control during quadrotor flight under turbulent environments \cite{NEURIPS2021_52fc2aee}. This makes Bayesian optimization suitable for this problem, as it tunes the parameters in a black-box manner. At the same time, the fact that the sensor data is quickly changing during flight, making it hard to distinguish any environment information. A well-designed approach needs to enable the optimization to obtain a PID that still achieves good performance on environments that are less seen in the training environment. In our experiment, we simulate the flight process of a quadrotor and use the PID control to fly the quadrotor along the set waypoints. We use Bayesian optimization to optimize the weights for PID control to minimize the averaged control errors (ACEs) defined as the mean squared error between the true coordinates and the desired waypoints. During the quadrotor simulation, the wind force is simulated with the widely-used Dryden model \cite{specification1980flying} to determine the wind speed in three directions at each time slot. We simulate two different wind domains. The first wind domain is for winds that have a mean value of zero and change more rapidly, and the second wind domain is for winds that have a non-zero mean value and are more stable. This is to test whether Bayesian optimization can find suitable PID weights to enable quadrotor to remain stable even under unseen environment.  Representing the wind force as a binary in the horizontal and vertical directions, the wind was experimentally distributed uniformly with a mean of (0,0) and a variance of (5,2.5) in the first domain, and with a mean of (3,1) and a variance of (2,1) in the second domain.

For comparison, we implement Bayesian optimization's surrogate model with DIL-GP, other GP variants, RF, and MLP. For GP-related methods,we use the upper confidence bound method as the acquisition function. The overall flow of the experiment is shown in Figure~\ref{bo_pic}. Bayesian optimization alternatively finds the next PID weights for trial and updates the surrogate model to predict the next PID weights for trial. Three flight trajectories obtained by PID optimized with different methods are shown in Figure~\ref{uav_result}. Other trajectories results can also be found in Appendix. The results shown are the mean results of five experiments, each run with a different random seed. The PID parameters obtained through DIL-GP Bayesian optimization demonstrates superior control performance on the quadrotor in unseen domain, with smaller control errors in the flight trajectory. The numerical comparison of all six methods is presented in Table~\ref{table:bo}.
As shown in Table~\ref{table:bo}, the variations in wind environments across the different experiments had a significant impact on the error values, resulting in larger deviations.
Due to limited OOD generalization abilities, baseline surrogate models struggles in predicting better solutions, while DIL-GP Bayesian optimization finds robust PID weights, with which the quadrotor can follow the desired paths more accurately. The error-iteration curves of ACE during Bayesian optimization are in Appendix for reference.

\section{Conclusion}
In this paper, we propose a domain invariant learning approach for Gaussian processes (DIL-GP) and its Bayesian optimization extension to improve their generalization abilities. Numerical experiments under challenging synthetic and real-world datasets demonstrate the effectiveness of the proposed min-max formulation of domain invariant learning for Gaussian processes. With the proposed framework, DIL-GP automatically discovers the heterogeneity in the data and achieves OOD generalization on various benchmarks. We further demonstrate that the Bayesian optimization algorithm with DIL-GP's superiority in a PID tuning problem.

\section*{Acknowledgments}
Nanyang Ye was supported in part by National Natural Science Foundation of China under Grant No.62106139, 61960206002, 62272301, 62020106005, 62061146002, 62032020, in part by Shanghai Artificial Intelligence Laboratory and the National Key R\&D Program of China (Grant NO.2022ZD0160100).

\bibliography{aaai24}

\newpage
\onecolumn
\appendix

\section{Experiment Details}
\label{supp:expdetails}
We apply GP, DIL-GP, GP with Rational Quadratic Kernel, GP with Dot Product Kernel, Random Forest, and MLP to the dataset. For a fair comparison, DIL-GP and GP both use the Gaussian kernel with the {Equation} \ref{GK}. In consideration of the relatively small sample size and to constrain the size of model parameters, a single hidden layer MLP is utilized. Rational Quadratic Kernel and Dot Product Kernel are represented by {Equation} \ref{RQK} and {Equation} \ref{DPK} respectively. The methods will be employed in the subsequent experiments detailed in the subsequent sections.
\begin{equation}
    K_G(X, Y)_{i, j} = s \exp{\frac{\|x_i - y_j\|}{l^2}},\label{GK}
\end{equation} 

\begin{equation}
    K_{RQ}(X, Y)_{i, j} = s \left(1 + \frac{\|x_i - y_j\|^2}{2\alpha l^2}\right)^{-\alpha},\label{RQK}
\end{equation}
    
\begin{equation}
   K_{DP}(X, Y)_{i, j} = s(x_i^T y_j + \sigma^2),  \label{DPK}
\end{equation}
where $x_i$ and $y_j$ is the $i$-th and $j$-th data item in the input.$s$,$l$,$\alpha$ and $\sigma$ are hyperparameters of the kernels.

For experiments on synthetic datasets and real-world datasets, We evaluate the performance using RMSE (Root Mean Square Error) as the metric. For GP-related methods, we defined an auxiliary metric Coverage Rate to indicate the proportion of test samples whose labels fall within the predicted range of $\mu \pm std$ by the model. This can be used in conjunction with RMSE to evaluate the validity of the model. We use the sklearn library for the implementation of MLP and RF. The specific structures of MLP and RF in the experiment are shown in Table~\ref{mlp_rf}.
\begin{table*}[!ht]
	\begin{center}
    {
        \begin{tabular}{l|c|c|c|c|c}
            \toprule 
             & 1-D Synthetic  & 2-D Synthetic &  Housing &  Mobile & BO\\
            \midrule
            RF(estimators)& 50&50  &50&50&50\\
            MLP(hidden layers)  &[64,64,64]&[64,64,64]   & [64,64,64] & [64,64,64] &[100]  \\

            \bottomrule
        \end{tabular}
    }
    \caption{Errors between the actual trajectory and the desired}
\label{mlp_rf}
\end{center}
\end{table*}

In different experiments, the settings of the trade-off parameter lambda are searched in [0.01, 0.025, 0.05, 0.1, 0.25, 0.5, 1, 1.5, 2, 3] and the best-performing value is selected respectively. In the training process, we use SGD optimizer because of its fast speed, and we do not observe performance gain by switching to other optimizers such as Adam or LBFGS optimizers in our experiment. The settings of learning rate are searched in [0.0001, 0.0005, 0.001, 0.005, 0.01, 0.03, 0.05, 0.1, 0.3, 0.5] and the best-performing value is selected respectively.

\subsection{Synthetic Dataset}
The specific implementation of the generation of the one-dimensional synthetic dataset is as follows. The dataset comprises two clusters of data with Gaussian distributions, which simulate two distinct domains. Specifically, the first cluster $X_1 \sim \mathcal{N}(0, 1)$, and the second cluster $X_2 \sim \mathcal{N}(6.5, 1)$. The labels of the two data clusters are generated by  $y_1 = 3\sin({x_1} / {2\pi}) + 3\epsilon$ and $y_2 = -\sin({(x_2 - 6.5)} / {32\pi}) + 0.5 + \epsilon$ respectively. The noise $\epsilon \sim  \mathcal{N}(0, 0.1)$. Our training set consists of one hundred samples from cluster $X_1$ and fifteen samples from cluster $X_2$ while the test set consists of eighty samples from cluster $X_2$. 

The specific implementation of the generation of the two-dimensional synthetic dataset is as follows. We sample two distinct clusters from Gaussian distributions. Specifically, the first cluster $X_1 \sim \mathcal{N}(\begin{bmatrix}0.3 \\ 0.3\end{bmatrix}, \begin{bmatrix}0.01 & 0 \\ 0 & 0.01\end{bmatrix})$, and the second cluster $X_2 \sim \mathcal{N}(\begin{bmatrix}0.7 \\ 0.7\end{bmatrix}, \begin{bmatrix}0.01 & 0 \\ 0 & 0.01\end{bmatrix})$.  The labels of the two data clusters are generated by the formulas $y_1 = 1.5\sin(30x_{11}+20)+1.5\sin(30x_{12}+20)+\epsilon_1$ and $y_2 = 0.5\sin(50x_{21}+20) + 0.5\sin(50x_{22}+20) +1.1 + \epsilon_2$ respectively. The noise $\epsilon_1 \sim  \mathcal{N}(\begin{bmatrix}0 \\ 0\end{bmatrix},\begin{bmatrix}0.1 & 0 \\ 0 & 0.1\end{bmatrix})$ and $\epsilon_2 \sim  \mathcal{N}(\begin{bmatrix}0 \\ 0\end{bmatrix},\begin{bmatrix}0.05 & 0 \\ 0 & 0.05\end{bmatrix})$. The training set consists of one hundred samples from cluster $X_1$ and fifteen samples from cluster  $X_2$, while the test set solely comprises eighty samples from cluster  $X_2$.

The setting of synthetic dataset generates a domain-shift environment, with the training set comprised mostly of data from cluster 1, and the test data selected from cluster 2, which generates a typical non-i.i.d distribution shifts scenario. There is a minor overlap between these two clusters, which makes the regression problem more challenging. The rooted mean square error (RMSE) is used as the comparison metric.

\subsection{Real-world Datasets}
In the two real world datasets, we scale the inputs and outputs to zero mean and unit standard deviation.  For simplicity, we set the number of latent environments to be 2. The optimal parameters for covariance kernels are selected using maximum likelihood estimation criterion with gradient descent.  We use cross validation to train the method, and the results are then tested on the test set. 

\begin{figure}[ht!]
\begin{subfigure}[b]{0.49\textwidth}
    \includegraphics[width=\linewidth]{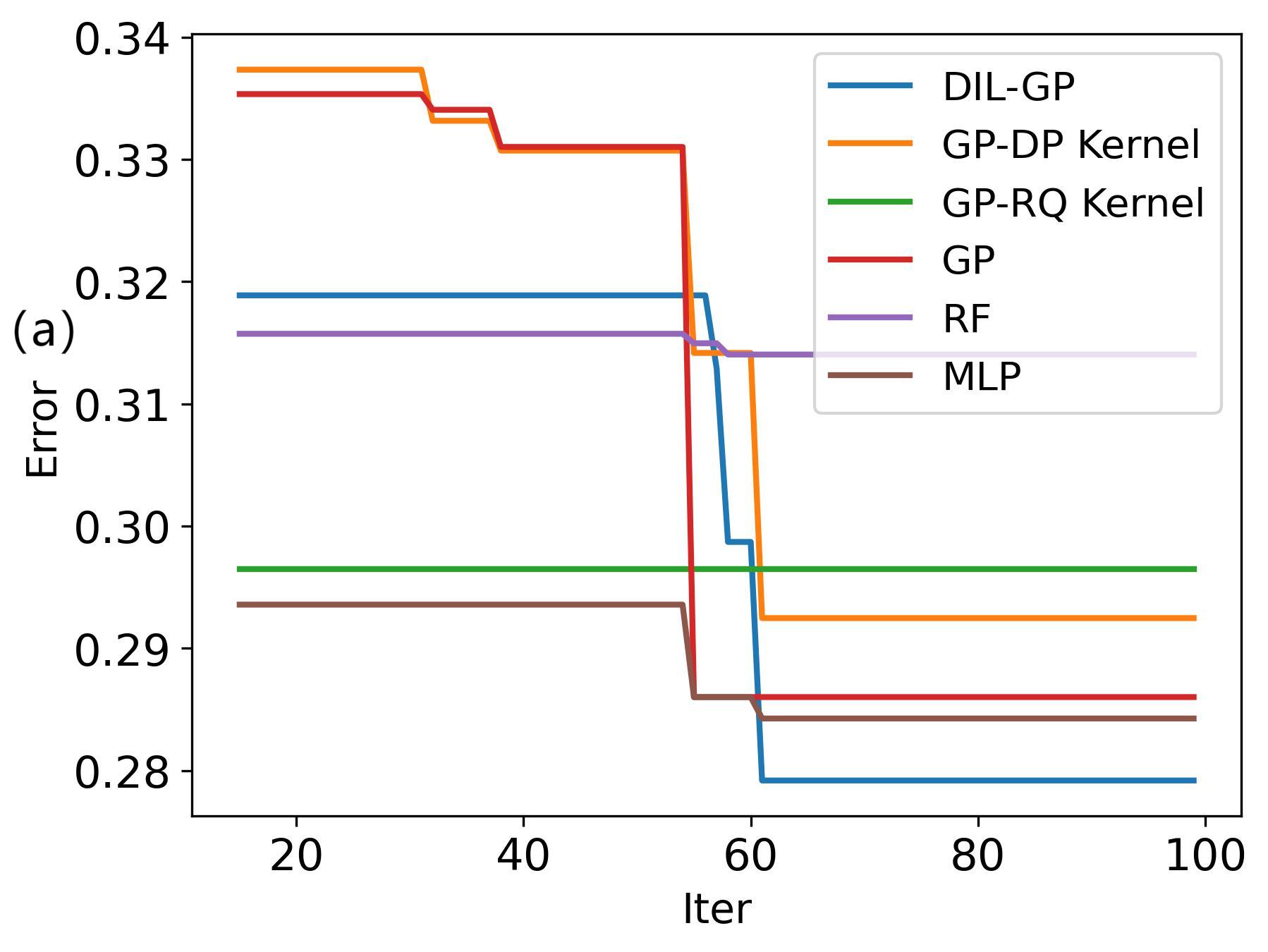}
\end{subfigure}
\begin{subfigure}[b]{0.49\textwidth}
    \includegraphics[width=\linewidth]{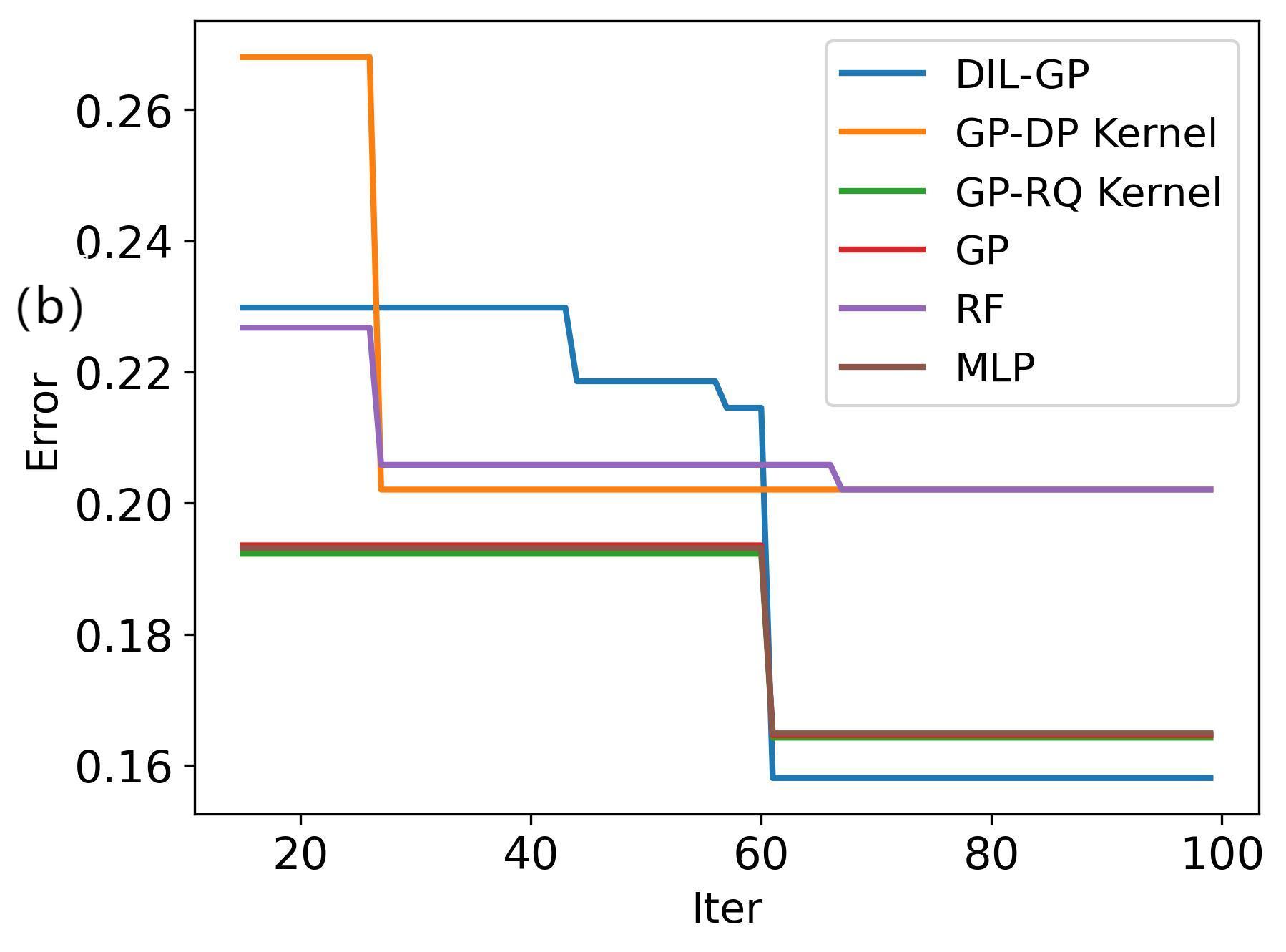}
\end{subfigure}
\caption{The evolution curves of PID control error during Bayesian optimization, which are obtained on the Hover and Fig-8 trajectories experiments. Methods all reach convergence in 100 epochs.}
\label{bo_curve}
\end{figure}
\subsection{Bayesian Optimization for Quadrotors PID
Tuning}
In the Bayesian Optimization section, we conducted simulation experiments of quadrotors. A quadrotor is a type of multirotor aircraft consisting of four rotors and a control system. Each rotor is driven by an electric motor and can generate lift and control the attitude of the aircraft by varying its speed and rotation direction. Quadrotors typically use inertial measurement units (IMUs), barometers, GPS, and other sensors to measure the attitude, altitude, and position of the aircraft, and use control algorithms to adjust the rotor speeds and rotation directions to achieve flight and control. The Proportional-Integral-Derivative (PID) control is a widely used feedback control algorithm that has been applied in quadrotors. This control algorithm employs the weighted feedback of the errors, consisting of three components: proportional, integral, and derivative of errors.  The weights for the three components need to be carefully tuned to achieve stable control during quadrotor flight under turbulent environments. In our experiment, we simulate the flight process of a quadrotor and use the PID control to fly the quadrotor along the waypoints. We use Bayesian optimization to optimize the weights for PID control to minimize the averaged control errors (ACEs) defined as the mean squared error between the true coordinates and the desired waypoints. In the experiments, we simulate two different wind domains. The first wind domain mainly consists of mild and steady wind magnitude, while the second mainly simulate turbelent outdoor environments. This is to test whether Bayesian optimization can find suitable PID weights to enable quadrotor to remain stable even under unseen environment.

The evolution curves of control error during Bayesian optimization are shown in Figure~\ref{bo_curve}, where we can observe faster convergence with DIL-GP as the surrogate model. For example, in Figure~\ref{bo_curve} (a), Bayesian optimization with DIL-GP achieves much lower error after only 60 steps. On the contrary, though MLP achieves lower errors at the initial 50 steps, it is stuck in the local minima and fail to predict better trial points. Due to space issues, in the main text we show experimental results for three trajectories, and in Figure ~\ref{uav_result2} we provide experimental results for all four trajectories.

\begin{figure*}[ht!]
  \centering
  \captionsetup[subfigure]{justification=centering} 

\rotatebox{90}{
\begin{minipage}[c]{.2\textwidth} 
    \centering 
    \caption*{Hover} 
  \end{minipage}%
}
      \begin{subfigure}[b]{0.2\textwidth}
      \caption*{DIL-GP}
    \includegraphics[width=\linewidth]{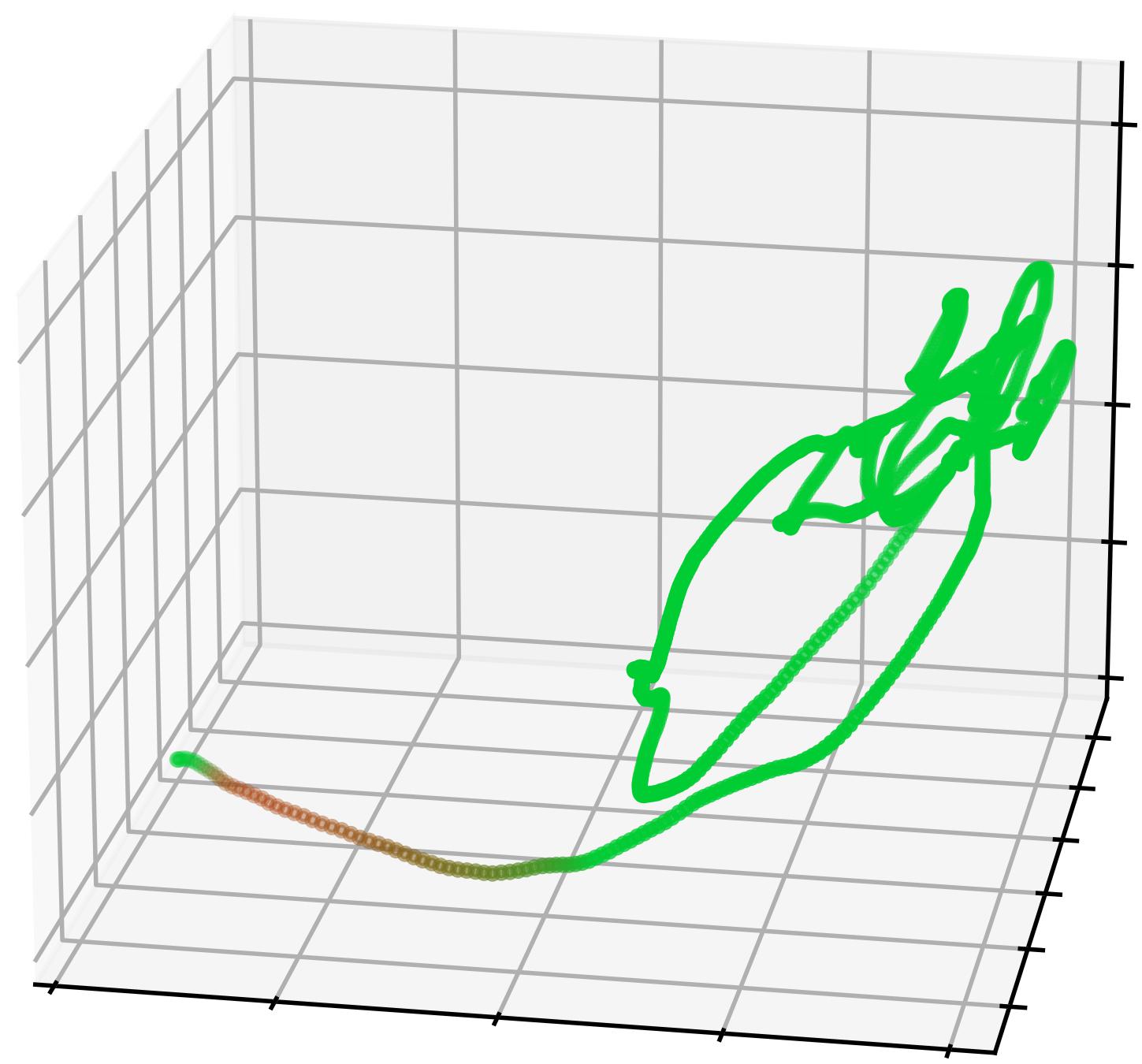}
  \label{fig:1b}
  \end{subfigure}
  \begin{subfigure}[b]{0.2\textwidth}
  \caption*{GP}
    \includegraphics[width=\linewidth]{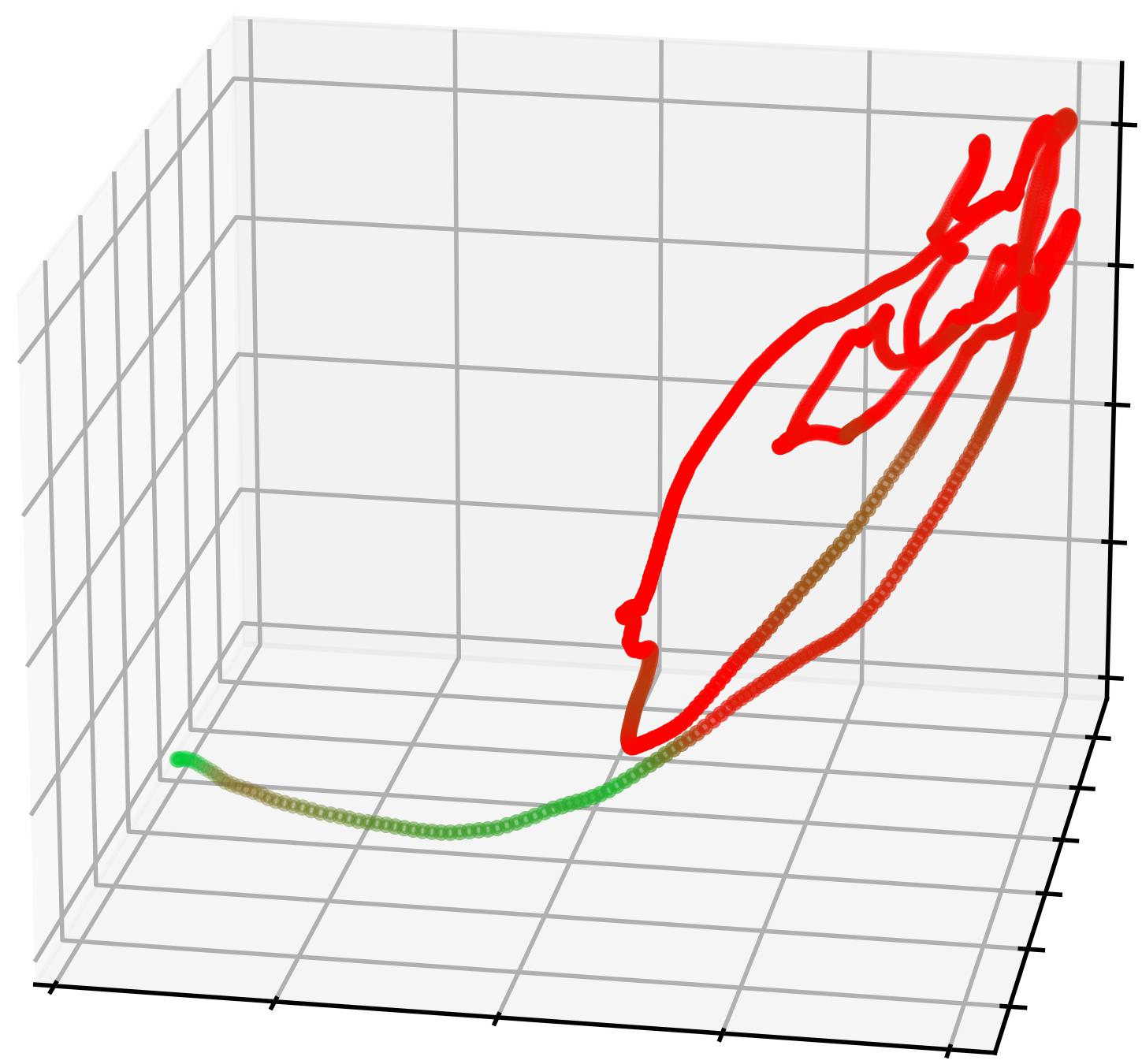}
  \end{subfigure}
    \begin{subfigure}[b]{0.2\textwidth}
    \caption*{GP-RQ Kernel}
    \includegraphics[width=\linewidth]{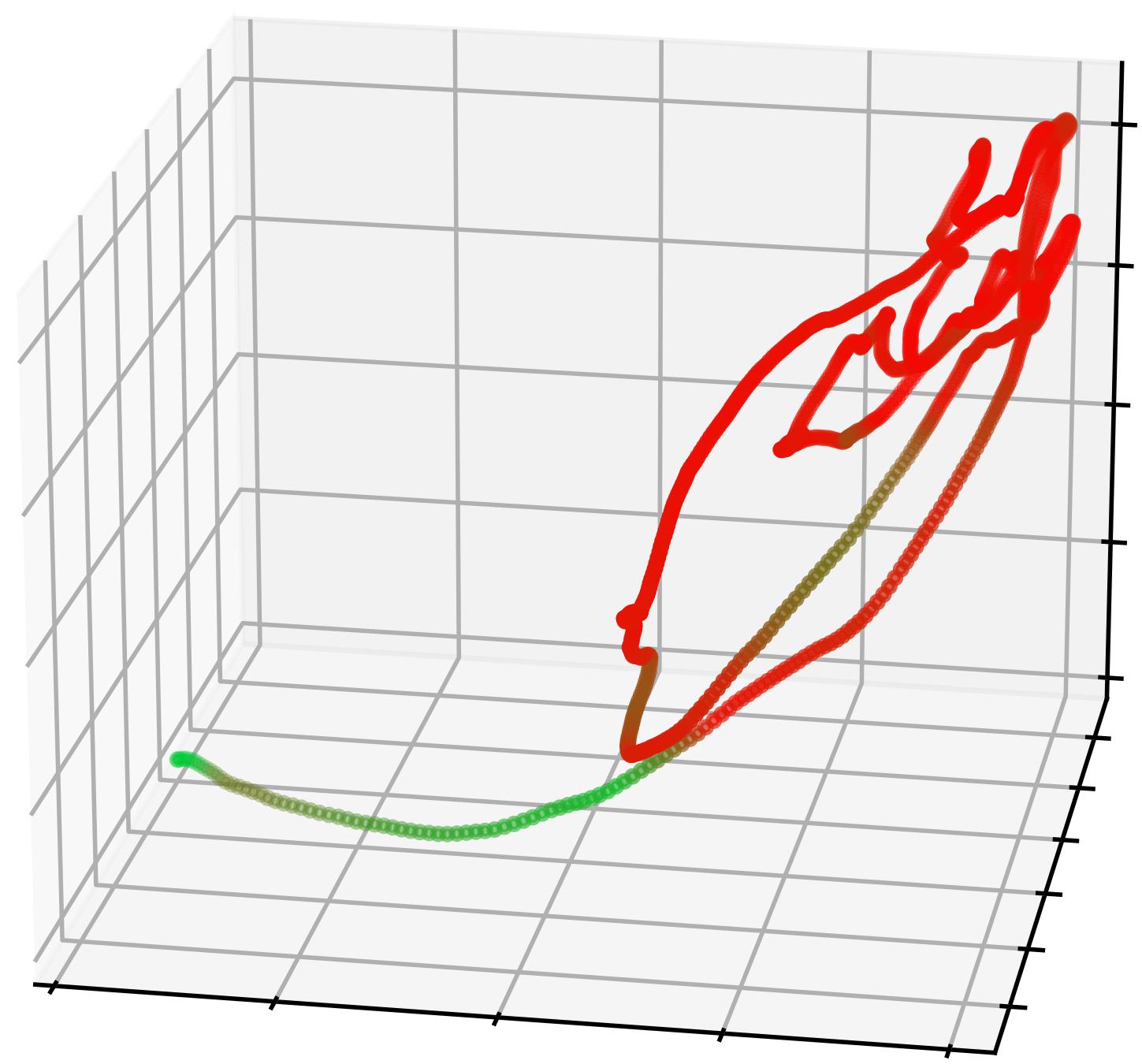}
  \end{subfigure}
    \begin{subfigure}[b]{0.2\textwidth}
    \caption*{GP - DP Kernel}
    \includegraphics[width=\linewidth]{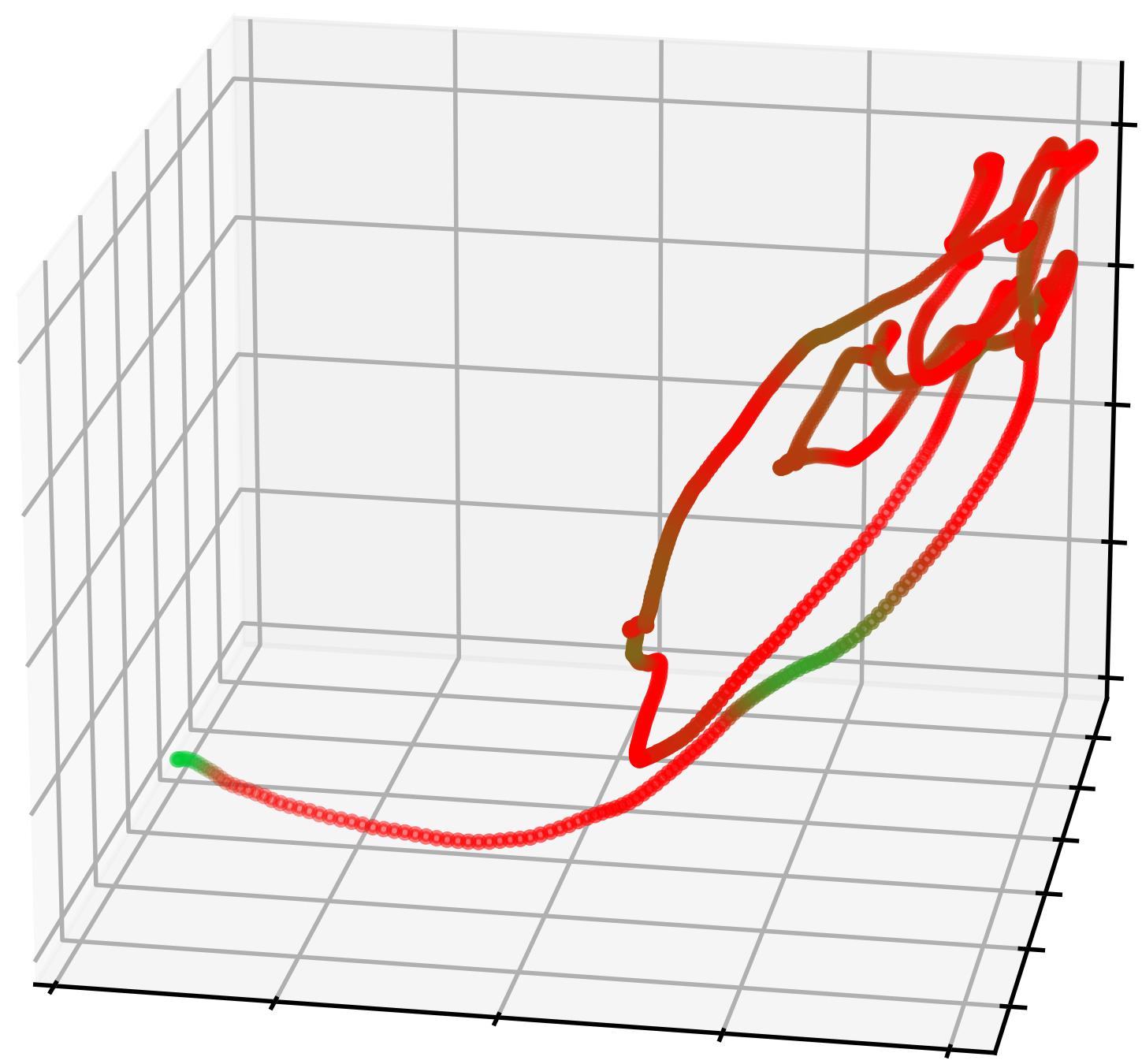}
  \end{subfigure}

  \rotatebox{90}{
\begin{minipage}[c]{.2\textwidth} 
    \centering 
    \caption*{Fig-8} 
  \end{minipage}%
}
    \begin{subfigure}[b]{0.2\textwidth}
    \includegraphics[width=\linewidth]{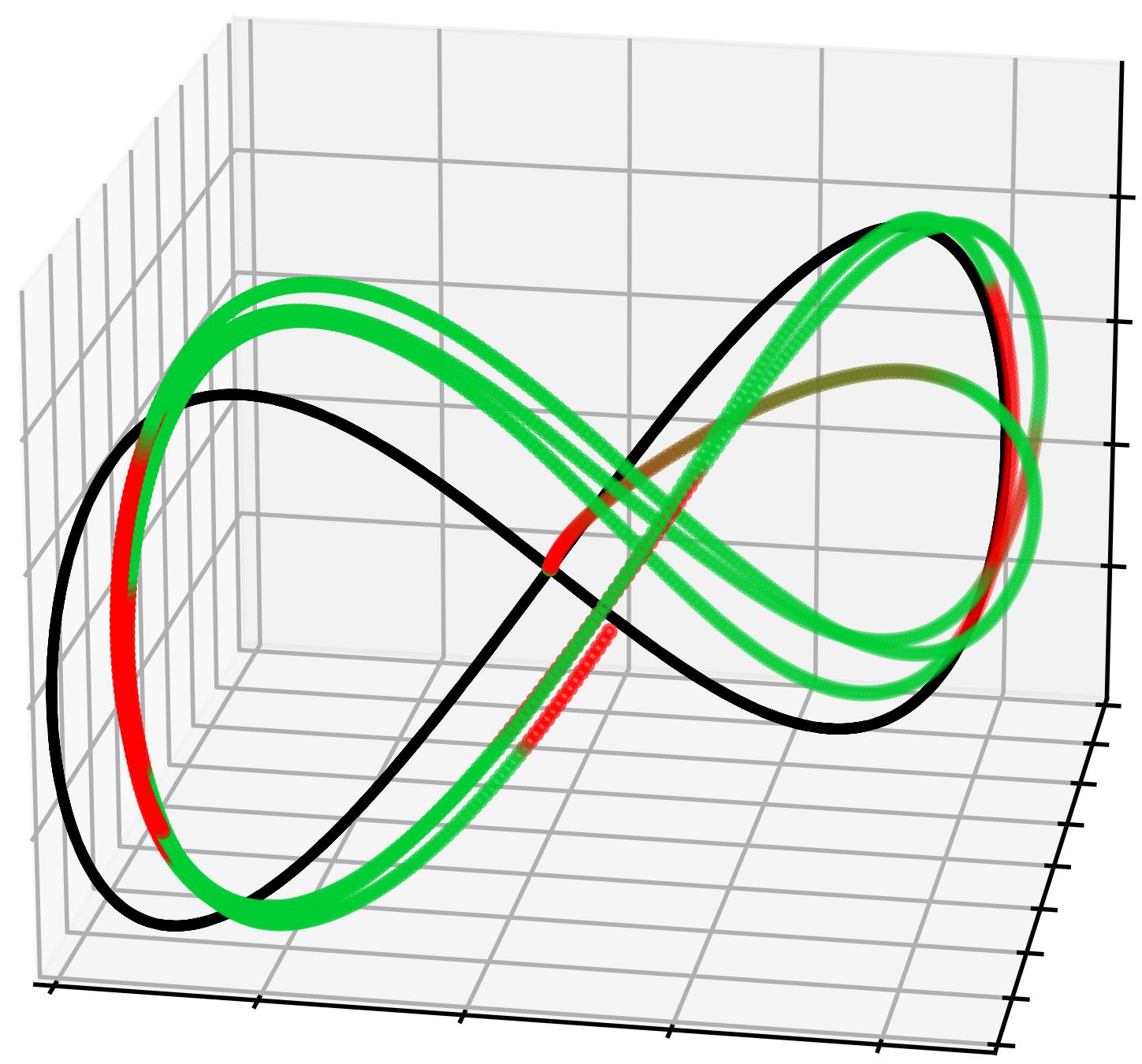}
  \end{subfigure}
    \begin{subfigure}[b]{0.2\textwidth}
    \includegraphics[width=\linewidth]{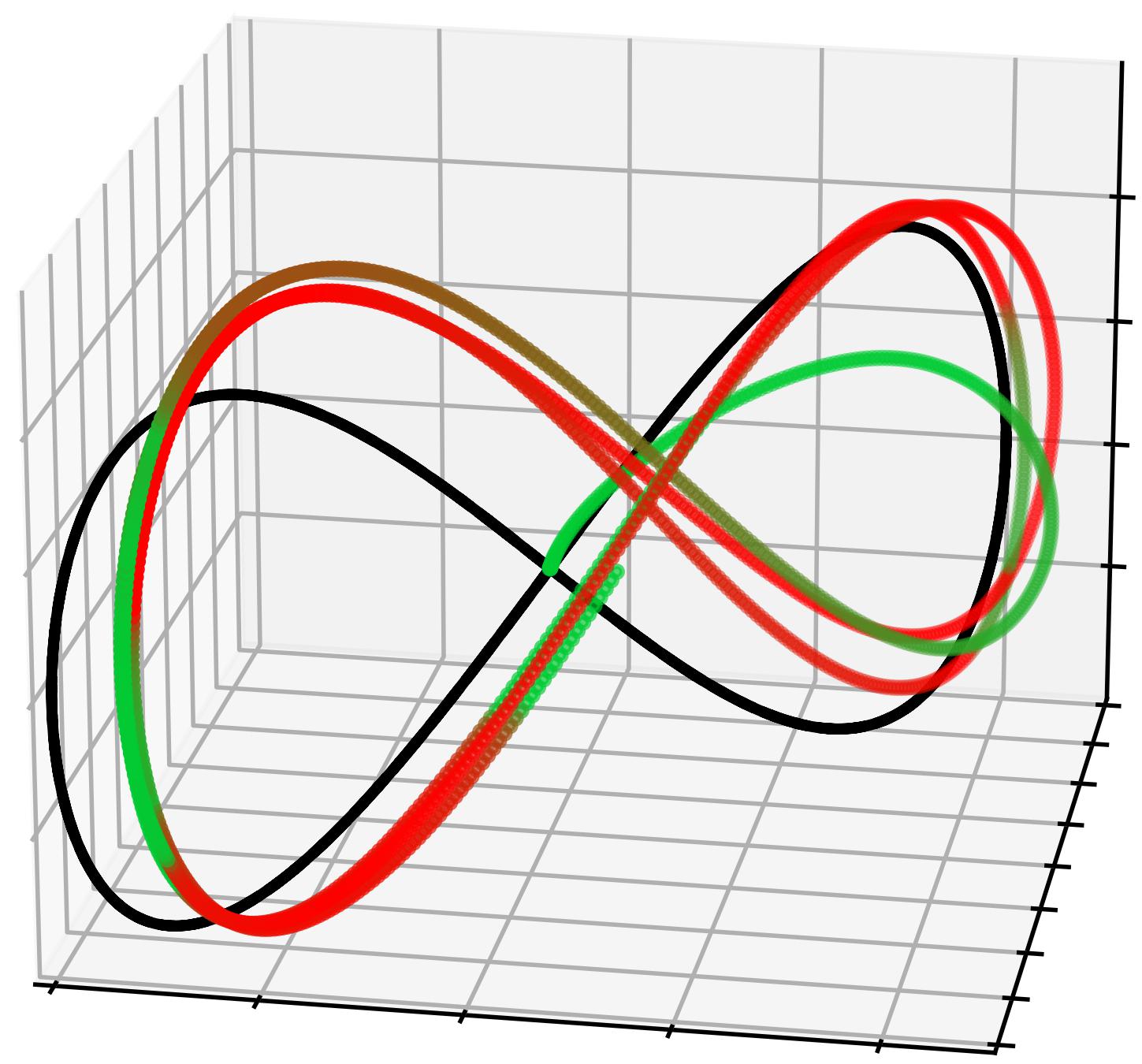}
  \end{subfigure}
    \begin{subfigure}[b]{0.2\textwidth}
    \includegraphics[width=\linewidth]{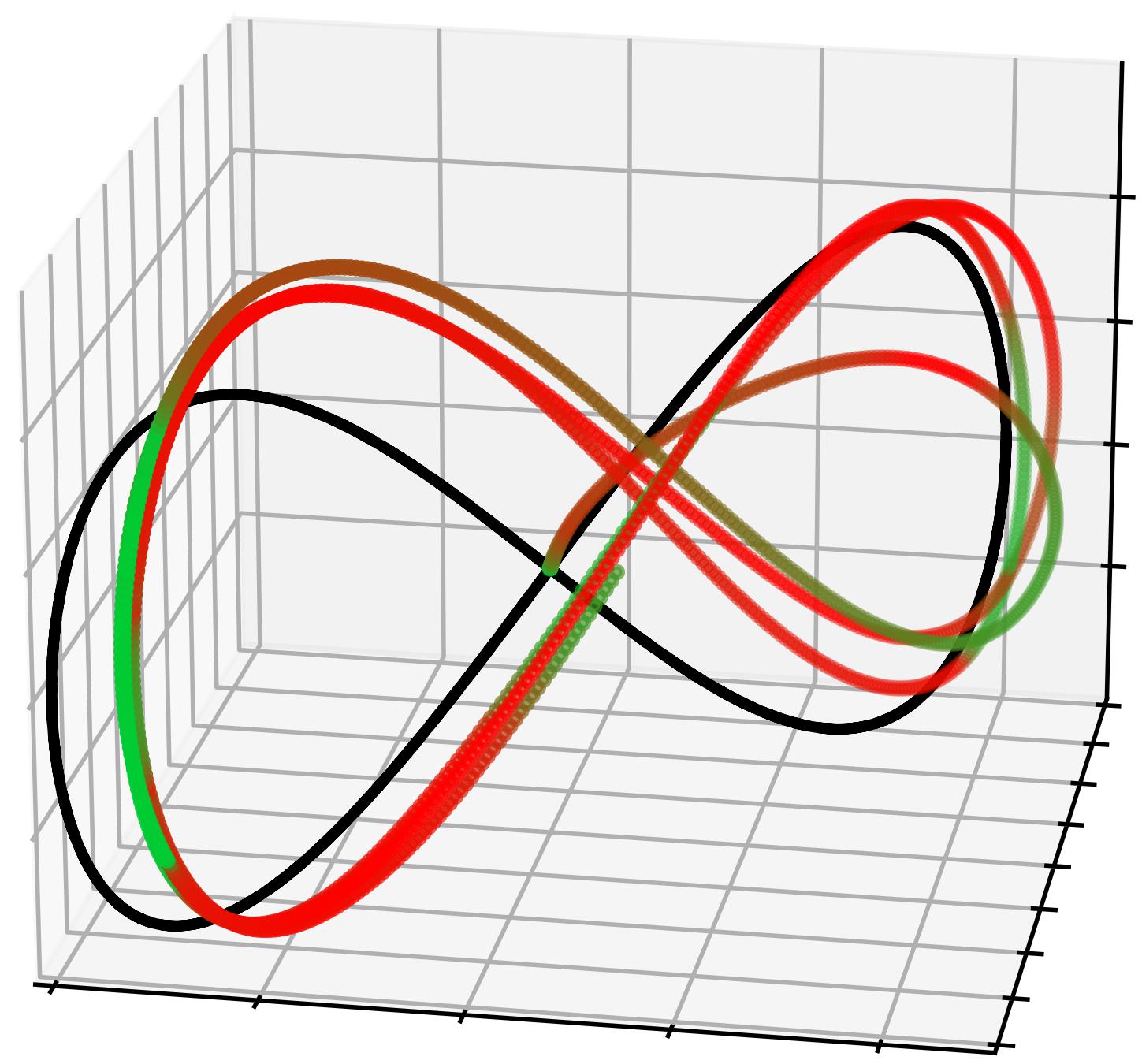}
  \end{subfigure}
    \begin{subfigure}[b]{0.2\textwidth}
    \includegraphics[width=\linewidth]{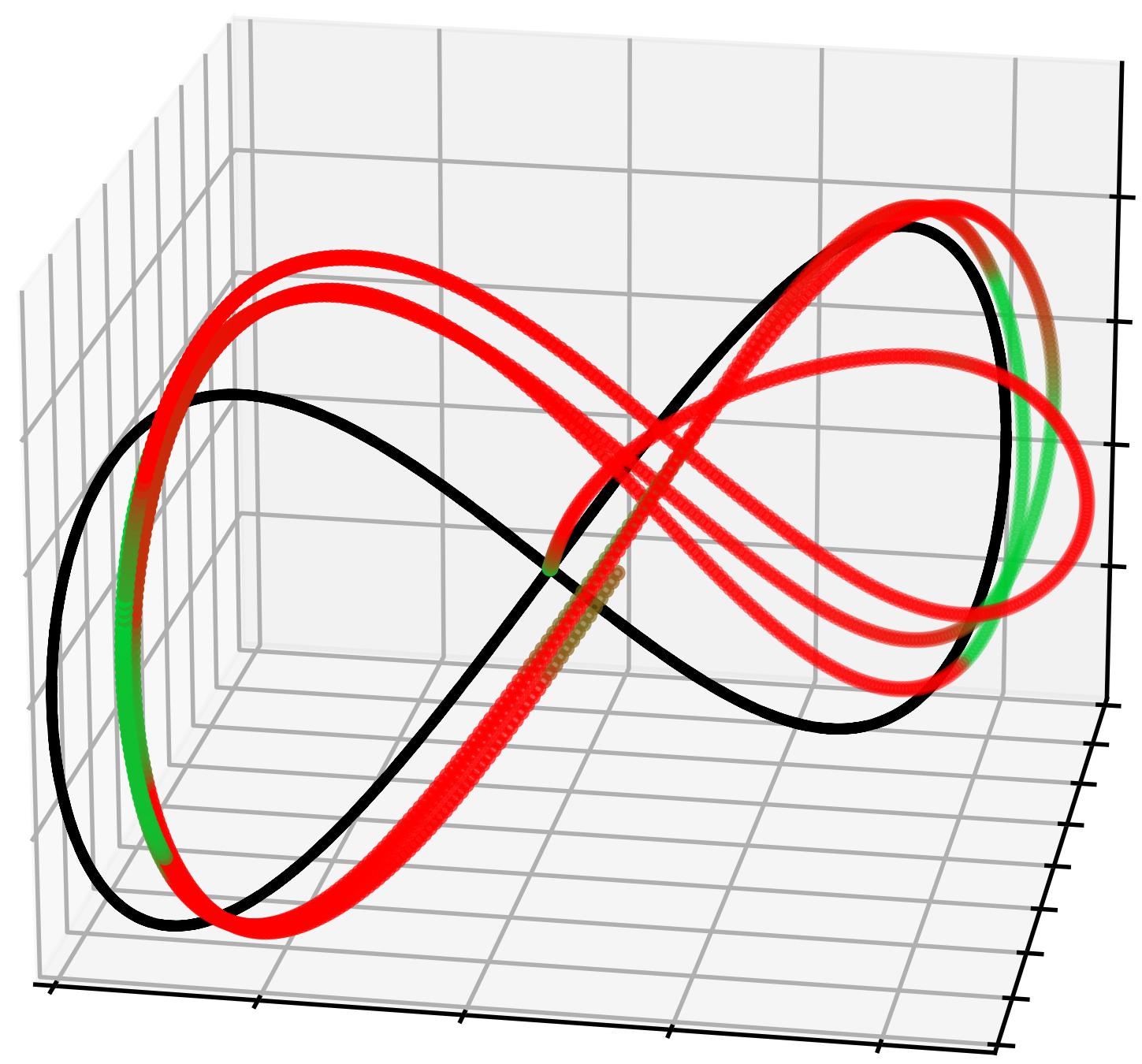}
  \end{subfigure}

  \rotatebox{90}{
\begin{minipage}[c]{.2\textwidth} 
    \centering 
    \caption*{Sin-foward} 
    \label{fig:side:caption} 
  \end{minipage}%
}
    \begin{subfigure}[b]{0.2\textwidth}
  
    \includegraphics[width=\linewidth]{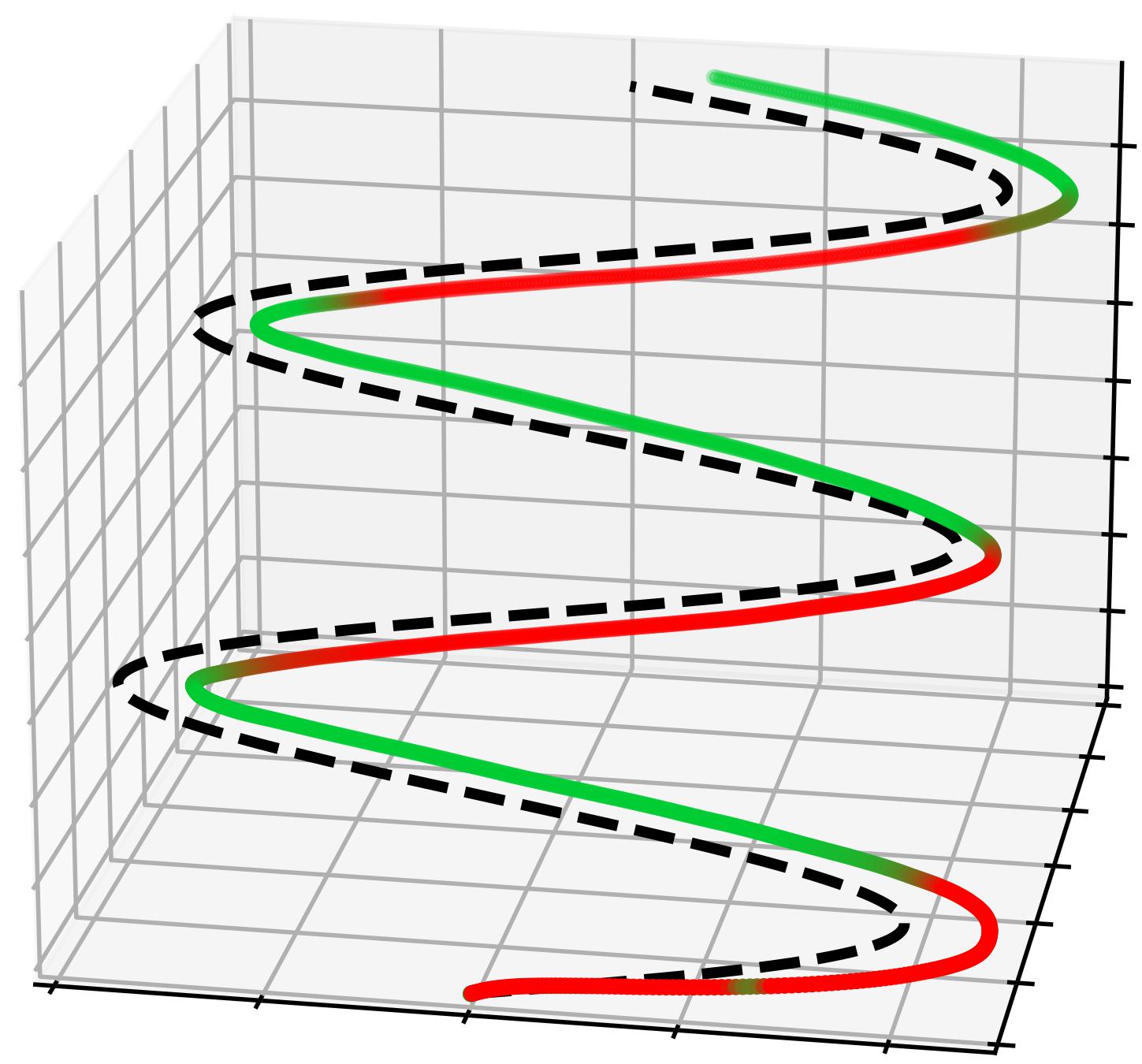}
  \end{subfigure}
    \begin{subfigure}[b]{0.2\textwidth}
    \includegraphics[width=\linewidth]{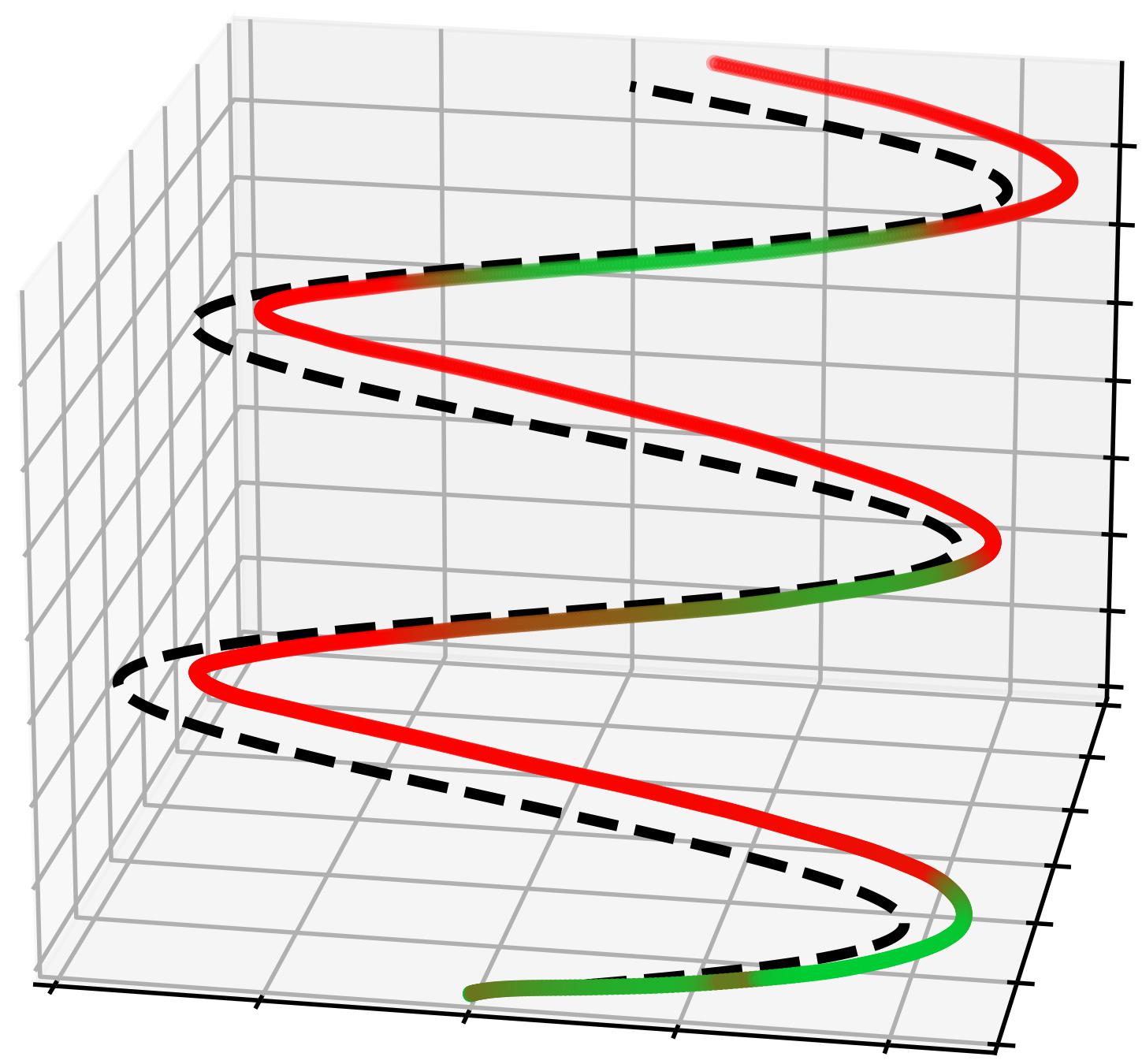}
  \end{subfigure}
      \begin{subfigure}[b]{0.2\textwidth}
    \includegraphics[width=\linewidth]{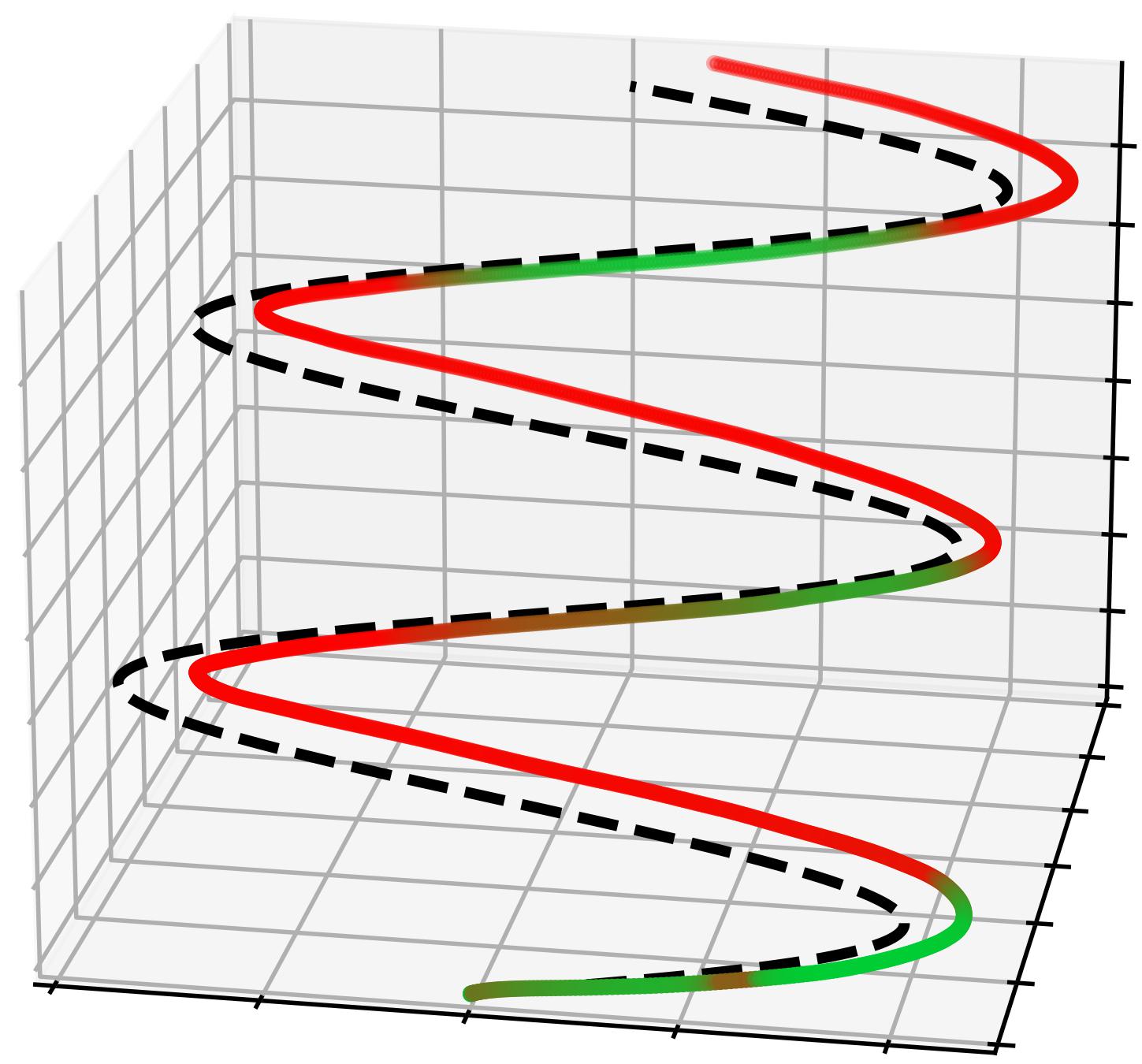}
  \end{subfigure}
      \begin{subfigure}[b]{0.2\textwidth}
    \includegraphics[width=\linewidth]{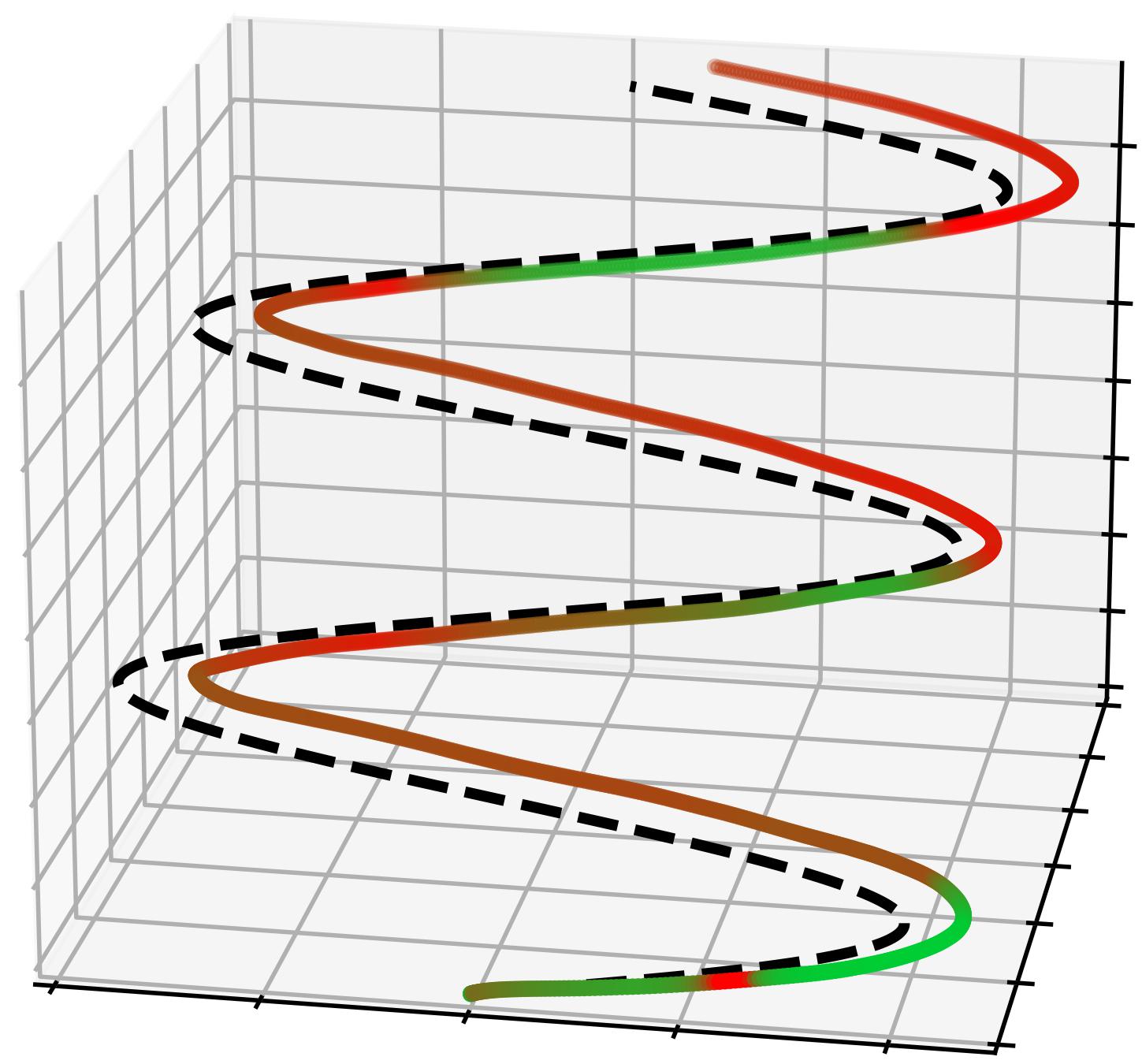}
  \end{subfigure}

  \rotatebox{90}{
\begin{minipage}[c]{.2\textwidth} 
    \centering 
    \caption*{Spiral-up} 
  \end{minipage}%
}
      \begin{subfigure}[b]{0.2\textwidth}
    \includegraphics[width=\linewidth]{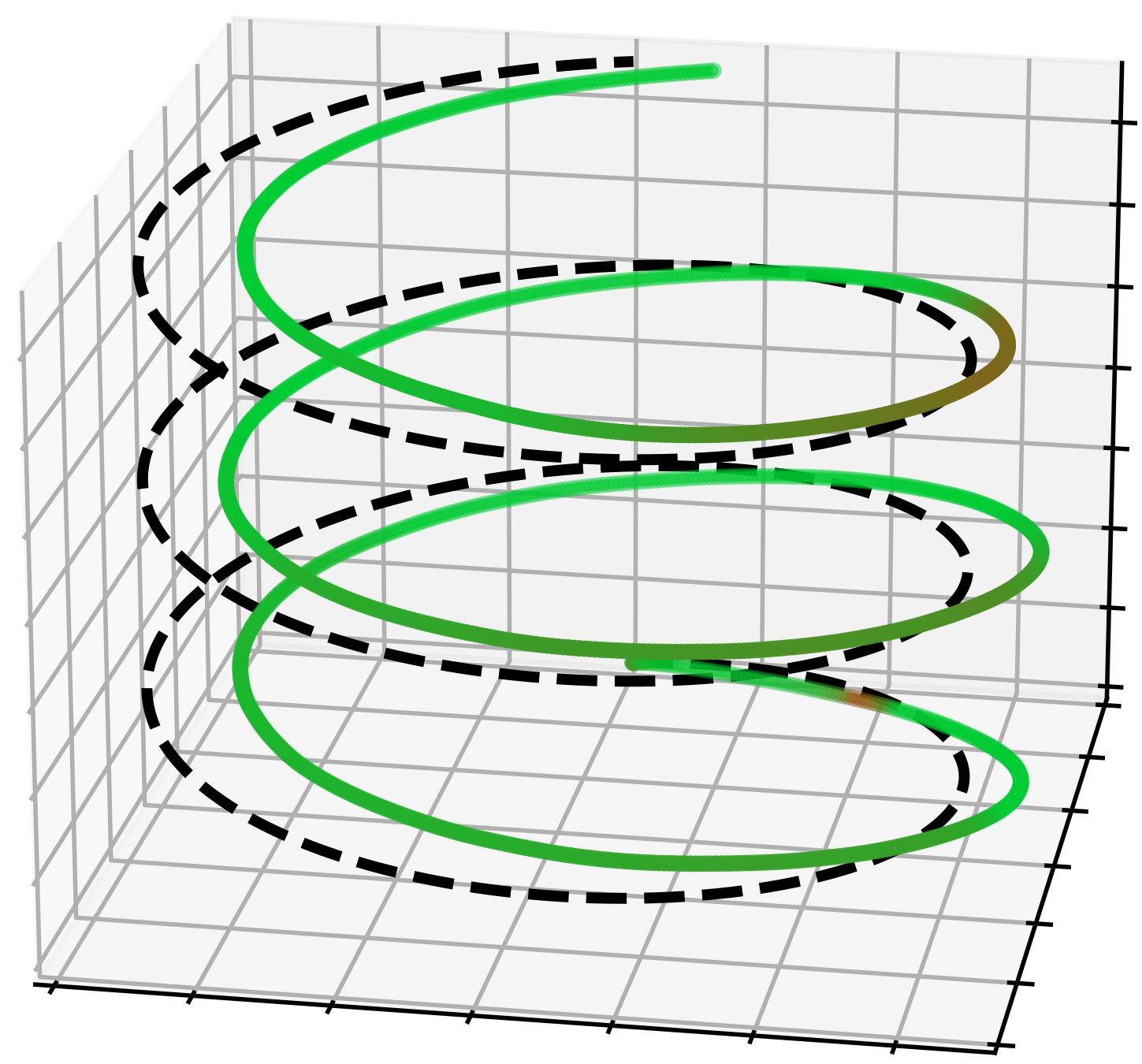}
  \end{subfigure}
      \begin{subfigure}[b]{0.2\textwidth}
    \includegraphics[width=\linewidth]{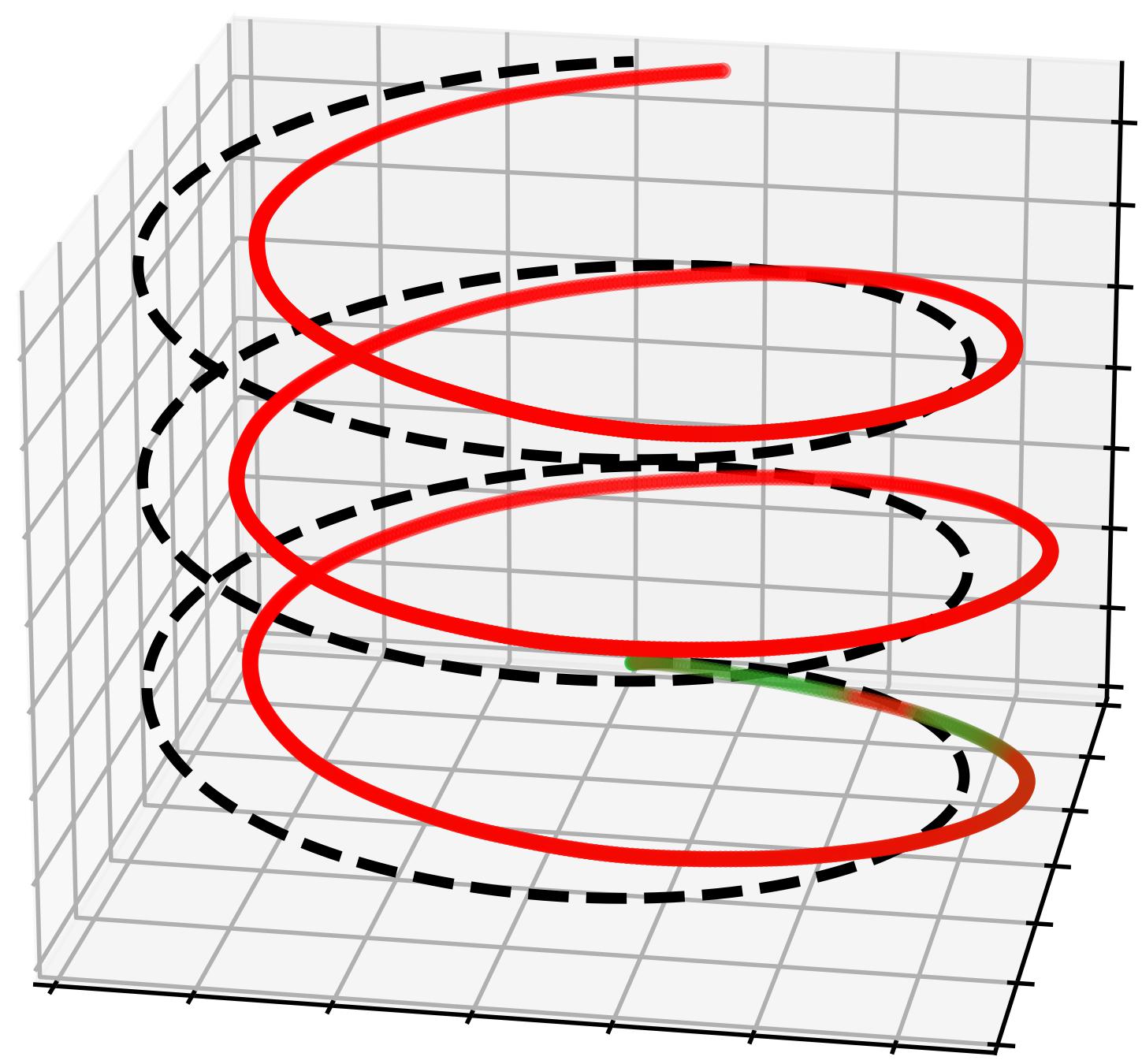}
  \end{subfigure}
      \begin{subfigure}[b]{0.2\textwidth}
    \includegraphics[width=\linewidth]{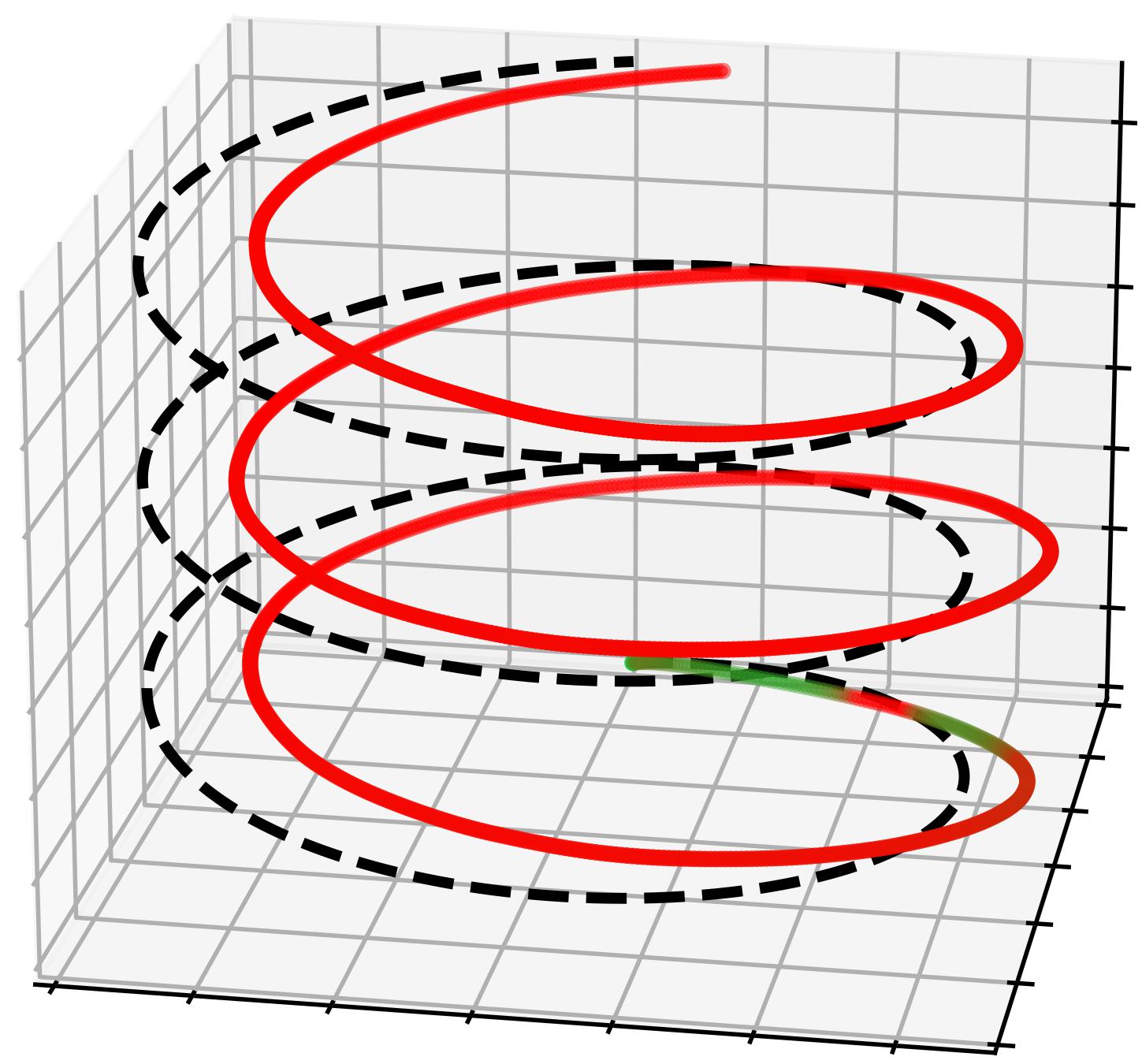}
  \end{subfigure}
      \begin{subfigure}[b]{0.2\textwidth}
    \includegraphics[width=\linewidth]{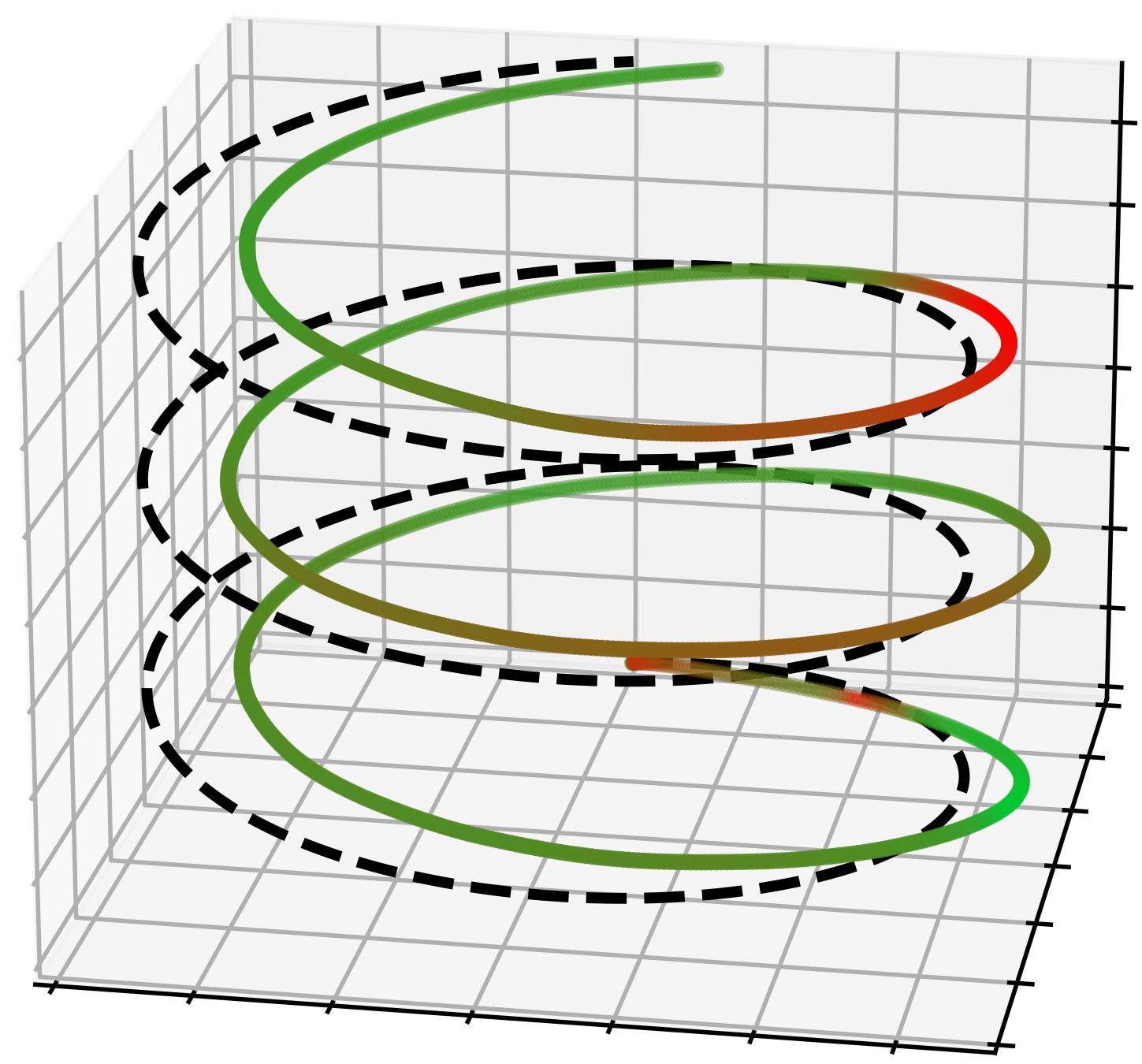}
  \end{subfigure}
  \caption{Visualizations of the quadrotor flight simulations on various trajectories. The black line is the desired trajectory. The green line denotes lower deviation from the desired trajectory while the red line denotes farther distances. BO using DIL-GP finds the best PID parameters, which enable accurate controls across all trajectories types under challenging turbulent wind conditions.}
  \label{uav_result2}
\end{figure*}

\clearpage
\section{Theoretical Analysis}

In this section, we provide theoretical analysis of the proposed DIL-GP. We begin by deriving the bound for the OOD generalization performance of DIL-GP. For the sake of simplicity and without loss of generality, we assume that the labels are generated from the target function: $y=f(x)+\epsilon$, where $\epsilon \sim N(0, \sigma^2)$.

\subsection{Generalization error optimization of DIL-GP}

It is reasonable to assume that $f \in L^2$, which is 
indispensable for the following proof. On this basis, Mercer's theorem allows us to decompose the kernel function on the basis $\phi_i(x)$: $k_{\theta}(x_i, x_j) = \sum_{p=1}^{\infty} \lambda_p \phi_p(x_i) \phi(x_j)$, where $\phi_i(x)$ are the eigenfunctions of the operator $Lf(x)=\int k(x, s)f(s)d\rho(s)$with corresponding eigenvalues $\lambda_1 \geq \lambda_2 \geq\ldots \geq 0$. This set also forms a basis, on which we can decompose any function $f(x)$, i.e. $f(x)=\sum_{p=0}^{\infty} \mu_p\phi_p(x)$, where $<\phi_0(x), \phi_i(x)> = 0, \forall i > 0$. 
We make the following assumption,
\begin{assumption}
1)The eigenvalue $\lambda_p$ is bounded, where $C_1$ and $C_2$ are constants,
    \begin{align}
C_1 p^{-\alpha} \leq \lambda_p \leq C_2 p^{-\alpha}, \forall p \geq 1,
0 \leq C_1 \leq C_2, \alpha \geq 1
\end{align}
2)The function decomposition coefficient $\mu_p$ is bounded by the following inequality,
\begin{align}
|\mu_p| \leq C_1 p^{-\beta}, \forall p \geq 1
|\mu_p| \geq C_2 p^{-\beta}, \forall p \geq 1, 
\beta > 1/2, C_1, C_2 > 0
\end{align}
3)$\sigma^2$ and $n^t$ are of the same order:
\begin{align}
\sigma^2=\Theta(n^t), 1-\frac{\alpha}{1+2\tau}<t<1
\end{align}
4)The eigenfunction set $\{\phi_i\}_{i=1}^\infty$ belongs to the union of eigenfunction sets

\begin{align}
\{\phi_i\}_{i=1}^\infty\subset \mathop{\cup}\limits_{e=0}^1 \{\phi_i^e\}_{i=1}^\infty
\end{align}
, where $\mathop{\cup}\limits_{e=0}^1 \{\phi_i^e\}_{i=1}^\infty$ denotes the function basis across environments that is perceivable by DIL-GP, and $\{\phi_i\}_{i=1}^\infty=\{\phi_i^0\}_{i=1}^\infty$ denotes the function basis within a single environment that is perceivable by vanilla GP methods.
\end{assumption}

We denote $\muB=(\mu_0, \ldots, \mu_i, \ldots)$, $\hat{\muB}_{1:R}=(\mu_1, \ldots, \mu_R, 0, \ldots, 0)$, $\phi_p(x)=(\phi_p(x_1), \ldots, \phi_p(x_n))^T$, $\Phi = (\phi_0(\xB), \ldots, \phi_p(\xB), \ldots)$, $\eta = (\eta_0(x^*), \ldots, \eta_p(x^*), \ldots)$, $\Lambda = \{0, \lambda_1, \ldots, \lambda_p, \ldots\}$, $K_n=\Phi \Lambda \Phi^T$, $f(\xB)=\Phi \mu$

\begin{lemma}
The expectation of square loss function between $\mu(x^*)$ and $f(x^*)$ is equivalent to the square of residual eigenvalue of the function decomposition, excepting for a infinitesimal remainder: 

\begin{align}
    &\mathds{E}_{\epsilon} \mathds{E}_{x^*\sim \tilde{p}(x)} [(\mu(x^*)-f(x^*))^2]=\mu_0^2+o(1)
\end{align}
\end{lemma}
\textit{Proof.}
\begin{align}
&\mathds{E}_{\epsilon} \mathds{E}_{x^*\sim \tilde{p}(x)} [(\mu(x^*)-f(x^*))^2]\\
&=\mathds{E}_{\epsilon} \mathds{E}_{x^*\sim \tilde{p}(x)}(\mu(x^*)-f(x^*))^2\\
&=\mathds{E}_{\epsilon} \mathds{E}_{x^*\sim \tilde{p}(x)}[\KB_{\theta}^{*x} (\KB_{\theta}^{xx} + \sigma^2 I)^{-1}y - f(x^*)]^2\\
&=\mathds{E}_{\epsilon} \mathds{E}_{x^*\sim \tilde{p}(x)}[\eta^T \Lambda \Phi^T[\Phi \Lambda \Phi^T + \sigma^2 I]^{-1}(\Phi \mu + \epsilon) - \eta^T \mu]^2\\
&=\mathds{E}_{\epsilon} \mathds{E}_{x^*\sim \tilde{p}(x)}[\eta^T \Lambda \Phi^T[\Phi \Lambda \Phi^T + \sigma^2 I]^{-1}\epsilon]^2+\mathds{E}_{x^*\sim \tilde{p}(x)}[\eta^T (\Lambda \Phi^T[\Phi \Lambda \Phi^T + \sigma^2 I]^{-1}\Phi-I)\mu]^2\\
&=[\sigma^2 Tr \Lambda \Phi^T(\Phi \Lambda \Phi^T+\sigma I)^{-2} \Phi \Lambda + \mu^T(I+\sigma^{-2}\Phi^T \Phi \Lambda)^{-1}(I+\sigma^{-2}\Lambda \Phi^T \Phi)^{-1} \mu]\\
&=Tr(I+\frac{\Lambda \Phi^T \Phi}{\sigma^2})^{-1}\Lambda-Tr(I+\frac{\Lambda \Phi^T \Phi}{\sigma^2})^{-2}\Lambda+
||(I+\frac{1}{\sigma^2}\Lambda \Phi^T \Phi)^{-1}\mu||^2
\end{align}
The following part discusses the two parts of the equation separately. To give the boundary of these values, it is viable to analysis the residual errors of truncated expansions. Following this thought, we have
\begin{align}
&||(I+\frac{1}{\sigma^2}\Lambda \Phi^T \Phi)^{-1}\muB|| \geq ||(I+\frac{1}{\sigma^2}\Lambda \Phi^T \Phi)^{-1}\hat{\muB}_{1:R}||-||(I+\frac{1}{\sigma^2}\Lambda \Phi^T \Phi)^{-1}(\muB-\hat{\muB}_{1:R})||\\
&||(I+\frac{1}{\sigma^2}\Lambda \Phi^T \Phi)^{-1}\muB|| \leq ||(I+\frac{1}{\sigma^2}\Lambda \Phi^T \Phi)^{-1}\hat{\muB}_{1:R}||+||(I+\frac{1}{\sigma^2}\Lambda \Phi^T \Phi)^{-1}(\muB-\hat{\muB}_{1:R})||
\end{align}
derives 
\begin{align}
&||(I+\frac{1}{\sigma^2}\Lambda \Phi^T \Phi)^{-1}\hat{\muB}_{1:R}||\\
&=\Theta(n^{(1-t)max\{-1, \frac{1-2\beta}{2\alpha}\}}log(n)^{k/2})\\
&=(1+o(1))||(I+\frac{n}{\sigma^2}\Lambda)^{-1}\hat{\muB}_{1:R}||
\end{align}
Finally, with probability of at least $1-\delta$, 
\begin{align}
&||(I+\frac{1}{\sigma^2}\Lambda \Phi^T \Phi)^{-1}\muB||\\
&=O(n^{max{t-1, \frac{(1-2\beta)(1-t)}{2\alpha}}}log(n)^{k/2})\\
&=(1+o(1))||(I+\frac{n}{\sigma^2}\Lambda)^{-1}\hat{\muB}_{1:R}||\\
&=(1+o(1))||(I+\frac{n}{\sigma^2}\Lambda)^{-1}\muB_R||
\end{align}
, where

\begin{equation}
k= \left \{
\begin{array}{cc}
0, 2\alpha&\neq 2\beta-1\\
1, 2\alpha&= 2\beta-1
\end{array}
\right.
\end{equation}

On the other hand, with probability of at least $1-\delta$, 
\begin{align}
&Tr(I+\frac{\Lambda \Phi^T \Phi}{\sigma^2})^{-1}\Lambda-Tr(I+\frac{\Lambda \Phi^T \Phi}{\sigma^2})^{-2}\Lambda\\
&=(Tr(I+\frac{n}{\sigma^2} \Lambda)^{-1} \Lambda)(1+o(1))-||\Lambda^{1/2}(I+\frac{n}{\sigma^2}\Lambda)^{-1}||^2(1+o(1))\\
&=\Theta(n^{\frac{(1-\alpha)(1-t)}{\alpha}})
\end{align}

Combing the two parts of equations, we achieve the following conclusion:
\begin{align}
&\mathds{E}_{\epsilon} \mathds{E}_{x^*\sim \tilde{p}(x)} [(\mu(x^*)-f(x^*))^2]\\
&=Tr(I+\frac{\Lambda \Phi^T \Phi}{\sigma^2})^{-1}\Lambda-Tr(I+\frac{\Lambda \Phi^T \Phi}{\sigma^2})^{-2}\Lambda+||(I+\frac{1}{\sigma^2}\Lambda \Phi^T \Phi)^{-1}\mu||^2\\
&=\Theta(max(\sigma^2 n^{\frac{1-\alpha-t}{\alpha}}, n^{\frac{(1-2\beta)(1-t)}{\alpha}}))
\end{align}
The performance of DIL-GP can be clearly illustrated in the following theorem:

\setcounter{theorem}{0}

\begin{theorem}
 Under mild assumptions,  DIL-GP's  OOD risk is strictly no larger than vanilla GP's OOD risk with probability $\geq 1-\delta$. Given $\delta \in (0, 1)$, $R_{\text{DIL-GP}}= \mathds{E}_{x^*}(\mu_{\text{DIL-GP}}(x^*)-f(x^*))^2 \leq R_{\text{GP}}= \mathds{E}_{x^*}(\mu_{\text{GP}}(x^*)-f(x^*))^2$, with probability $\geq 1-\delta$.
\label{theorem3}
\end{theorem}

\textit{Proof.}
According to Lemma 1, with
probability of at least 1-$\delta$ we have
\begin{align}
\mathds{E}_{(x^*, y^*)}(\mu(x^*)-y^*)^2-\sigma^2 = \mathds{E}_{x^*}(\mu(x^*)-f(x^*))^2 =\Theta (n^{\frac{(1-\alpha)(1-t)}{\alpha}}) + \mu_0^2 + o(1) = \mu_0^2 + o(1)
\end{align}

For a random function in the Hilbert space $L^2$, the kernel function generated by DIL-GP ensures more detailed decomposition of the function, producing a larger norm of the function. The environment enhancement of DIL-GP facilitates more efficient training outcomes. The algorithm ensures that the eigenfunctions of $K_\theta$ capture additional information from the two distributions of input data. 
Therefore, function f can be expressed by the series of function basis, i.e.
\begin{align}
f=\sum\limits_{e=0}^1 \sum\limits_{i=1}^\infty\lambda_i^e \phi_p^e+\lambda_0 \phi_0
\end{align}

The function basis$ \{\phi_i^1\}_{i=1}^\infty$ is equivalent to the function basis $\{\phi_i\}_{i=1}^\infty$
On this basis, the richer function basis considered by DIL-GP generates the following result:
\begin{align}
R_{\text{DIL-GP}}= \mathds{E}_{x^*}(\mu_\text{DIL-GP}(x^*)-f(x^*))^2=\lambda_0^2 \leq \lambda_0^2 + \sum\limits_{i=1}^\infty ({\lambda_i^1})^2= \mathds{E}_{x^*}(\mu_\text{GP}(x^*)-f(x^*))^2
\end{align}

\section{Convergence proof of DIL-BO}

In this section, we present the convergence analysis of the proposed domain invariant learning for Bayesian optimization. Bayesian optimization aims to find the parameter for maximizing an unknown function $f: \Dcal \rightarrow \Rcal$. In each iteration $t=1 \cdots T$, an input $x_t \in \Dcal$ is queried to get the output $y_t = f(x_t)+\epsilon$ where $\epsilon \sim \Ncal(0, \sigma^2)$ where $\Ncal(0, \sigma^2)$ is a Gaussian noise with variance $\sigma^2$. The performance of DIL-BO is measured by the cumulative regret $R_{T}=\sum_{1\cdots T}[f(x_*)-f(x_t)]$, where $x_*=\argmax_{x \in \Dcal} f(x)$ is the global minimizer of $f(x)$. We want to give an upperbound on the cumulative regret $R_{T}$. We also assumes $f$ lies in the reproduced kernel Hilbert (RKHS) space associated with the DIL GP's kernel $k$ and the norm induced by this RKHS is bounded: $\|f\|_{k} \leq B$.

First, we will give some lemmas used in the proof of the theorem.
\begin{lemma}
Given $\delta \in (0,1)$ and $\beta_t=B+\sigma\sqrt{2(\gamma_{t-1}+1+\log(4/\delta)}$, then
\begin{equation}
\lvert f(x) - \mu_{t-1}(x)\rvert \leq \beta_t \sigma_{t-1}(x)
\end{equation}
which holds with probability $\geq 1-\delta/4$.
\label{lemma:f_bound}
\end{lemma}

This lemma directly follows Theorem 2 from \cite{Sayak2017}.

\begin{lemma}
Given $\delta \in (0,1)$, suppose the algorithm is run with parameters $\beta_t$, then with probability$\geq 1-3/4\delta$, for $t \geq 1$:
\begin{equation}
    r_t \leq 2\beta_t \sigma_{t-1}(x_t)
\end{equation}
\label{lemma:rtbound}
\end{lemma}
\begin{proof}
Given $\delta \in (0,1)$, the optimum regret $f(x^*)$ follows with probability$\geq 1 - \delta / 4$,
\begin{align}
f(x^*) \leq \mu_{t-1}(x^*) + \beta_t \sigma_{t-1}(x^*) \leq \mu_{t-1}(x) + \beta_t \sigma_{t-1}(x) 
\end{align}
The first inequality is due to Lemma~\ref{lemma:f_bound}. The second inequality is due to the policy for selecting $x_t$ in seeking maximum.
Then, the instantaneous regret $r_t$ is upper-bounded by
\begin{align}
r_t = f(x^*) - f(x_t) \leq \mu_{t-1}(x) + \beta_t \sigma_{t-1}(x) -f(x_t) \leq 2\beta_t \sigma_{t-1}(x_t)
\end{align}
\end{proof}

\begin{lemma}
Let $\mathbf{f_T}$ and $\mathbf{y_T}$ be the function values and noisy observations after $T$ iterations. Then, the information gain about $f$ from the first $T$ iterations is:
\begin{equation}
    I(\mathbf{y_T};\mathbf{f_T}) = \frac{1}{2} \sum_{t=1}^{T} \log{[1+\sigma^{-2}\sigma^2_{t-1}(x_t)]}
\end{equation}
\end{lemma}
This lemma directly follows Lemma 5.3 from \cite{Niranjan2009}.

\begin{lemma}
Suppose the BO is run with $\beta_t$, define the maximum information gain as $\gamma_T = \max_{A \in \Dcal, \lvert A \rvert =T} I(\mathbf{y_{A}}; \mathbf{f_{A}})$ in which $\mathbf{y_{A}}$ and $\mathbf{f_{A}}$ represent  observations and function values from a set $A$ of input sizes $T$. Then,
\begin{equation}
    \sum_{t=1}^{T}[2\beta_t \sigma_{t-1}(x_t)]^2 \leq C_1 \beta^2_T \gamma_T
\end{equation}
where $C_1=\frac{8}{\log{1+\sigma^{-2}}}$.
\label{lemma:betat_ub}
\end{lemma}
\begin{proof}
Each element inside the summation can be upperbounded by,
\begin{align}
4\beta^2_t\sigma^2_{t-1}(x_t) \leq 4\beta^2_T  \sigma^2 \sigma^{-2} \sigma^2_{t-1}(x_t) \leq 4\beta^T  \sigma^2 [\frac{\sigma^{-2}}{\log{(1+\sigma^{-2})}} \log{(1+\sigma^{-2}\sigma^2_{t-1}(x_t))}]
\end{align}
The first inequality follows since $\beta_t$ is non-decreasing. The second inequality follows as $\sigma^{-2}m \leq \frac{\sigma^{-2}}{\log{(1+\sigma^{-2})}} \log{(1+\sigma^{-2}m)}$ for all $m \in (0,1]$ and $\sigma^2_{t-1}(x_t) \in (0,1]$. Then, by taking the summation,
\begin{align}
\sum_{t=1}^{T} 4\beta^2_t\sigma^2_{t-1}(x_t) \leq \beta^2_T \frac{8}{\log{(1+\sigma^{-2})}} \sum_{t=1}^{T} [\frac{1}{2}\log{(1+\sigma^{-2}\sigma^2_{t-1}(x_t))}] = \beta^2_T \frac{8}{\log{(1+\sigma^{-2})}} I(\mathbf{y_T};\mathbf{f_T}) \leq C_1 \beta^2_T \gamma_T
\end{align}
This is by the definition of $C_1$, $I(\mathbf{y_T};\mathbf{f_T})$, and $\gamma_T$.

\end{proof}

\begin{theorem}
\textbf{(Convergence of DIL-BO)} Given $\delta \in (0,1)$, denote $\gamma_t$ the maximum information gain after observing $t$ observations. If run the BO process with $\beta_t=B+\sigma\sqrt{2(\gamma_{t-1}+1+\log(4/\delta)}$, with probability$\geq 1-\delta/4$, the upperbound of the cumulative regret $R_T$ satisfies:
\begin{equation}
    R_T = \sum_{t=1}^{T} r_t \leq \beta_T \sqrt{C_1T\gamma_T}
\end{equation}
\label{theorem4}
\end{theorem}
\begin{proof}
\begin{align}
    R_T &= \sum_{t=1}^{T} r_t \leq \sum_{t=1}^{T} 2\beta_t \sigma_{t-1}(x_t) \leq \sqrt{T} \sqrt{\sum_{t=1}^{T} (2\beta_t \sigma_{t-1}(x_t))^2} \leq \sqrt{C_1 T \beta^2_T \gamma_T} = \beta_T \sqrt{C_1T\gamma_T}
\end{align}
The first inequality is due to Lemma~\ref{lemma:rtbound}. The second inequality is due to Cauchy-Schwartz inequality. The third inequality is due to Lemma~\ref{lemma:betat_ub}.
\end{proof}

\end{document}